\documentclass{article}

\usepackage{microtype}
\usepackage{graphicx}
\usepackage{subcaption}
\usepackage{booktabs} 

\usepackage{hyperref}

\usepackage[accepted]{icml2026}

\usepackage{amsmath}
\usepackage{amssymb}
\usepackage{mathtools}
\usepackage{amsthm}

\usepackage{multirow}

\usepackage{listings}
\usepackage{xcolor}
\lstdefinestyle{py}{
  language=Python,
  basicstyle=\ttfamily\small,
  keywordstyle=\color{blue},
  commentstyle=\color{gray!70},
  stringstyle=\color{teal},
  showstringspaces=false,
  breaklines=true,
  numbers=left, numberstyle=\tiny\color{gray},
}

\usepackage[capitalize,noabbrev]{cleveref}

\theoremstyle{plain}
\newtheorem{theorem}{Theorem}[section]

\newtheorem{corollary}[theorem]{Corollary}
\theoremstyle{definition}

\theoremstyle{remark}

\usepackage[textsize=tiny]{todonotes}

\icmltitlerunning{Critical Percolation as a Synthetic Data Model for Interpretability}

\begin{document}

\twocolumn[
  \icmltitle{Critical Percolation as a Synthetic Data Model for Interpretability}



  \icmlsetsymbol{equal}{*}

  \begin{icmlauthorlist}
      \icmlauthor{Aryeh Brill}{poi}
      \icmlauthor{Tom Ingebretsen Carlson}{poi}
  \end{icmlauthorlist}

  \icmlaffiliation{poi}{Principles of Intelligence}

  \icmlcorrespondingauthor{Aryeh Brill}{aryeh.brill@gmail.com}

  \icmlkeywords{Machine Learning, mechanistic interpretability, synthetic data, percolation theory, critical phenomena, random trees, additive coalescence, hierarchical structure, neural scaling laws, linear probes, fractals}
    
  \vskip 0.3in
]



\printAffiliationsAndNotice{}  

\begin{abstract}
    Neural networks learn features that reflect the hierarchical, multi-scale structure of natural data. Synthetic datasets used to evaluate interpretability methods typically lack this structure, limiting their value as realistic toy models. To close this gap, we introduce a family of synthetic datasets consisting of hierarchical functions defined on critical mean-field percolation clusters embedded in a high-dimensional data space. The percolation data consists of sparse, low-dimensional fractal clusters with a power-law size distribution. Latent variables modeling a taxonomic hierarchy generate each data point's target value. The data model is analytically tractable with known critical exponents that fix its properties without requiring hyperparameter tuning. We leverage a mapping between percolation clusters, random trees, and additive coalescence to propose an almost linear-time algorithm to jointly sample a random tree and its hierarchical latent decomposition, enabling data generation at arbitrary scale. Using probing experiments, we find that the model's ground-truth latent variables can be linearly decoded from neural network activations. Together, sparsity, self-similarity, power-law statistics, and analytical tractability make critical percolation a principled testbed for interpretability research.
\end{abstract}

\section{Introduction}\label{sec:introduction}

\begin{figure*}[h!]
    \centering
    \begin{subfigure}[b]{0.63\textwidth}
        \centering
        \includegraphics[width=\textwidth]{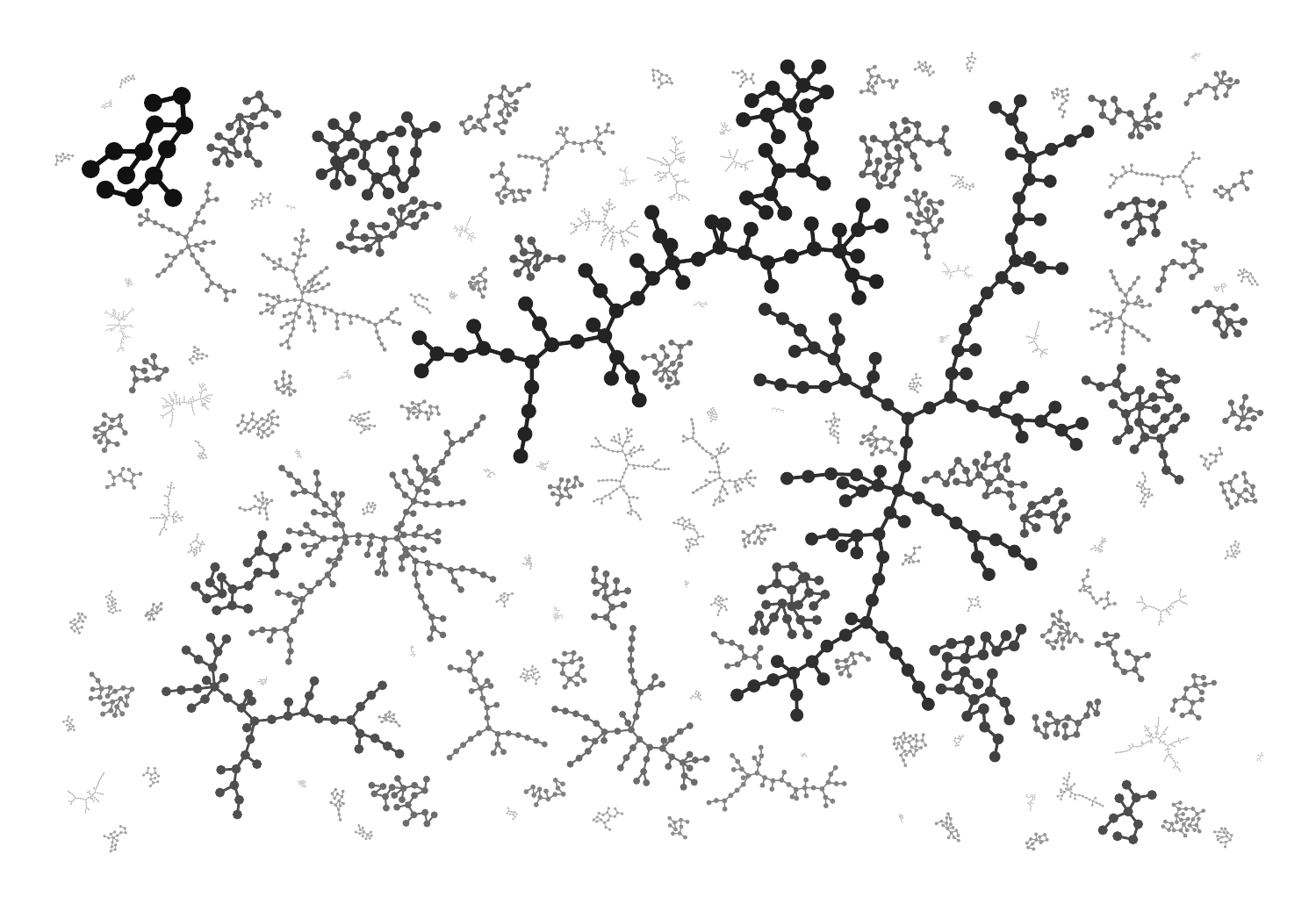}
        \caption{}
        \label{fig:data_model_illustration}
    \end{subfigure}
    \hfill
    \begin{subfigure}[b]{0.35\textwidth}
        \centering
        \includegraphics[width=\textwidth]{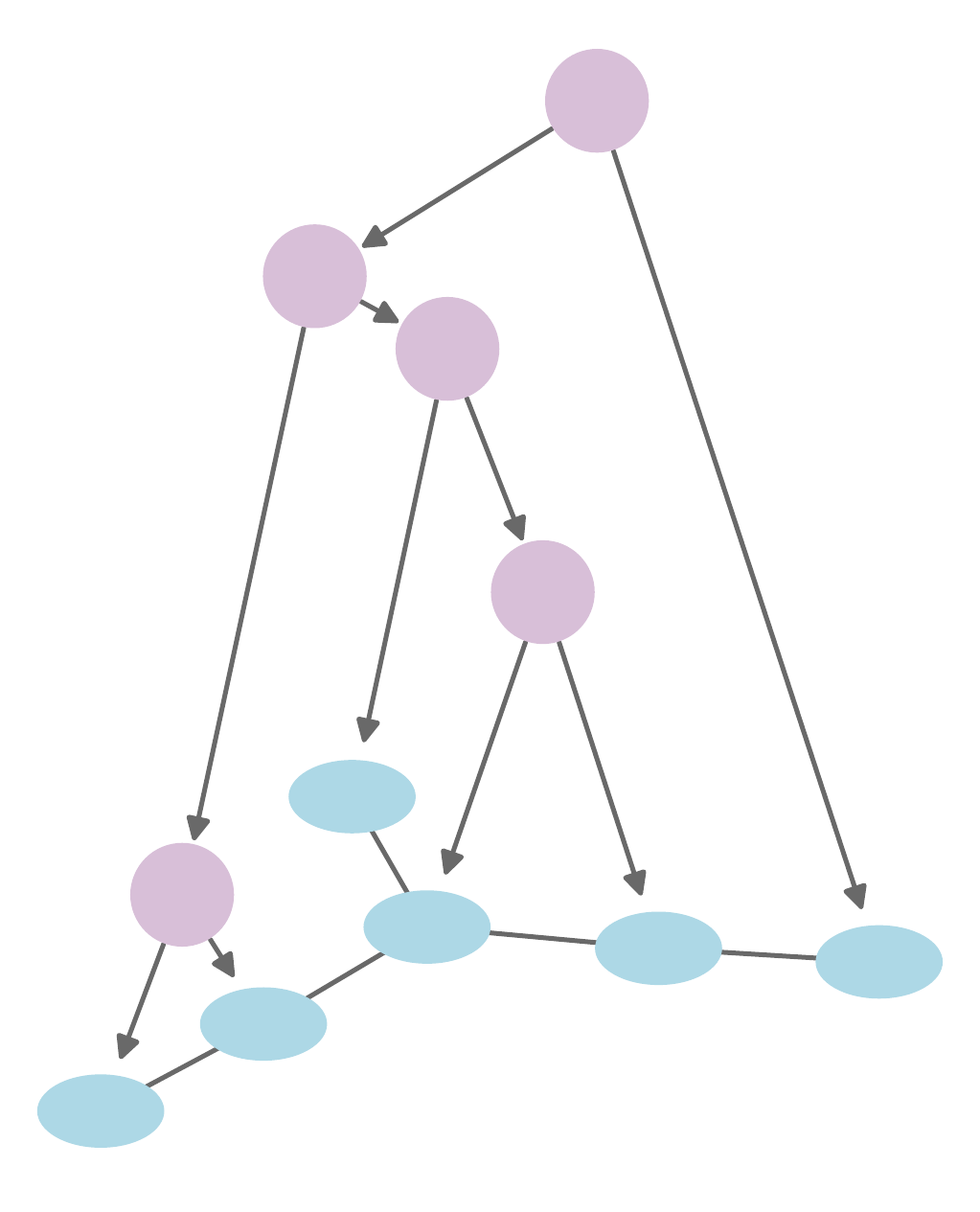}
        \caption{}
        \label{fig:hierarchy_illustration}
    \end{subfigure}
    \caption{The percolation data model. (\subref{fig:data_model_illustration}) Inputs are distributed as self-similar fractal clusters with power-law sizes. (\subref{fig:hierarchy_illustration}) Targets are generated by hierarchical latent variables decomposing each cluster.}
    \label{fig:figure1}
\end{figure*}

Deep neural networks have achieved extraordinary success across hard real-world tasks \citep{lecun2015deep, krizhevsky2012imagenet, brown2020language, jumper2021highly}. To efficiently achieve strong predictive performance, a learning system must represent features that match the relevant latent factors that explain the data \citep{bengio2013representation}. By developing tractable data models that capture common structural properties of natural data, we can better understand AI systems.

Data must have structure to be efficiently learnable, due to the curse of dimensionality \citep{bellman1961adaptive}. For example, approximating a generic $d$-dimensional continuous function with uniform error $\epsilon$ requires $O\left((1/\epsilon)^d\right)$ samples, which is intractable when $d$ is large. Since both humans and machines learn to perform real-world tasks, natural data distributions must be highly structured.

Researchers have invoked a variety of structural properties of data to explain why deep learning succeeds. One important perspective comes from mechanistic interpretability, which seeks to reverse-engineer the features and mechanisms learned by neural networks \citep{olah2020zoom, elhage2022toy, bereska2024mechanistic}. By decomposing dense representations into a large number of sparsely activating latent features, sparse autoencoders (SAEs) successfully extract interpretable features from the activations of large language models (LLMs) \citep{elhage2022toy, cunningham2023sparse, bricken2023monosemanticity, templeton2024scaling, gao2024scaling, lieberum2024gemma}. These methods succeed because the distribution of natural language tasks is \textit{sparse}: concepts relevant for prediction occur, and co-occur, rarely. The sparsity of learned features reflects the sparsity of data.

Besides sparsity, real-world concepts exhibit \textit{hierarchical} organization. Such hierarchies are either compositional, relating wholes to their parts, or taxonomic, relating classes to subclasses \citep{woods1975s, brachman1983and, miller1995wordnet}. In mechanistic interpretability research, hierarchical structure has been probed via the pathologies of SAEs. If data were purely sparse, increasing an SAE's size would simply recover additional interpretable latent features. In reality, previously coarse-grained latents split into fine-grained ones, a phenomenon known as feature splitting \citep{bricken2023monosemanticity}. Related issues are feature absorption and composition, in which at larger SAE sizes previously interpretable latents fragment into special cases or into compositional features, respectively \citep{chanin2024absorption, leask2025sparse}. These phenomena suggest that learned LLM features have hierarchical structure, reflecting the hierarchical organization of natural concepts. Indeed, SAE variants designed to recover hierarchically organized latents can achieve excellent reconstruction while recovering interpretable features at different levels of abstraction \citep{bussmann2025learning, costa2025flat}.

Compositional hierarchical structure in data can be modeled using a probabilistic context-free grammar (PCFG) \citep{allen2023physics, garnier2024transformers, lubana2024percolation, menon2025analyzing, cagnetta2024deep, cagnetta2024towards, cagnetta2025learning, sclocchi2025phase}. Recursive application of a PCFG’s production rules yields compositional hierarchical structure creating long-range correlations among the observed features or tokens. However, data models that describe taxonomic hierarchical structure have been comparatively less well studied.

Another influential idea is that natural datasets have low \textit{intrinsic dimension}. If all natural data samples are supported on a low-dimensional data manifold, representations can be highly compressed and the curse of dimensionality curtailed \citep{bengio2013representation, goldt2020modeling}. Many methods exist to estimate the intrinsic dimension of data embedded in a higher-dimensional ambient space \citep{grassberger1983measuring, levina2004maximum, facco2017estimating, binnie2025survey}.

Further insight comes from neural scaling laws, which suggest the presence of \textit{power laws} in data \citep{hestness2017deep, henighan2020scaling, kaplan2020scaling, hoffmann2022training}. Analyses based on high-dimensional regression typically model the kernel spectrum as a power law \citep{spigler2020asymptotic, bordelon2020spectrum, maloney2022solvable, bahri2024explaining, atanasov2024scaling}. Each of sparsity, hierarchy, and low intrinsic dimensionality can produce power-law scaling. If the distribution of sparse latent features or subtasks is heavy-tailed, learning them in order of frequency or importance translates into power-law scaling \citep{hutter2021learning, michaud2023quantization, nam2024exactly, pan2025understanding, brill2025model, michaud2025understanding}. Models of compositional hierarchical data predict data-limited scaling laws connected to the resolved context horizon
\citep{cagnetta2025learning, cagnetta2026deriving}. Finally, low intrinsic dimension naturally yields efficient power-law scaling with exponent $1/D$, where $D$ is the intrinsic dimension $D \ll d$ \citep{spigler2020asymptotic, sharma2022scaling, bahri2024explaining}.

A fractal model unites all of these properties. Fractal geometry is ubiquitous in nature \citep{mandelbrot1983fractal}. Fractals can be disconnected, modeling sparsity. Fractals are typically \textit{self-similar}, with the same properties at all scales, yielding hierarchical organization. A fractal fills space with a dimension distinct from both its topological dimension and its embedding dimension, reflecting intrinsically low dimensionality. Self-similar fractals are scale-free, naturally giving rise to power laws. Fractality reflects the intuition that nature's data-generating process is essentially inexhaustible. The more one learns, the more detailed distinctions one can make, and there's always more to learn.

Synthetic data models with realistic, principled structure are valuable for mechanistic interpretability research. Synthetic datasets provide ground-truth latent features, allowing researchers to validate and improve interpretability tools. Much progress has been spurred by synthetic data models with heavy-tailed sparse features \citep{elhage2022toy} and hierarchically organized features \citep{chanin2024absorption, bussmann2025learning, costa2025flat}. \citet{chanin2026synthsaebench} propose a synthetic benchmark model of neural activations that incorporates feature sparsity and hierarchy, as well as superposition and correlation. These synthetic datasets model the empirical characteristics of neural activations. To ground interpretability research in meaningful properties of the data itself, principled data models are needed. 

We introduce a synthetic data model based on critical mean-field percolation theory \citep{stauffer2018introduction}. We build on a model proposed by \citet{brill2024neural, brill2025representation}. The percolation model describes data that has sparsity, taxonomic hierarchy, power laws, low intrinsic dimension, and self-similarity. The data model is very simple, consisting of clusters of randomly occupied units on a high-dimensional lattice. Despite this simplicity, the model possesses rich structure, including critical phenomena, self-similarity, and fractal geometry. Figure~\ref{fig:figure1} illustrates the percolation model.

This work presents several contributions. First, we introduce an explicitly hierarchical construction of the percolation data, revealing its innate self-similarity by constructing target values in accordance with latent variables organized in binary trees. Second, we highlight a mapping between critical mean-field percolation clusters, random labeled trees \citep{moon1970counting}, and the additive coalescent process \citep{ aldous1998standard, pitman1999coalescent}, and leverage this mapping to propose a \textit{cyclic coalescent} algorithm to jointly sample a random tree and its hierarchical latent decomposition in almost linear $O(n~\alpha(n))$ time. Third, we develop tools to generate synthetic hierarchical percolation datasets using the cyclic coalescent. Fourth, we empirically investigate representations of neural networks trained on synthetic percolation data via probing experiments.\footnote{The code for synthetic data generation is available \href{https://github.com/aribrill/percolation-synthetic-data/tree/3dd47a808f02d161a34615570814f107ab1c1627}{here}. The code for our interpretability experiments is available \href{https://github.com/tomingebretsencarlson/mechinterp-percolation/tree/a5fd5846350c68ef5528374cbcfce266ed8e622d}{here}.}

\begin{figure}
    \centering
    \includegraphics[width=0.8\linewidth]{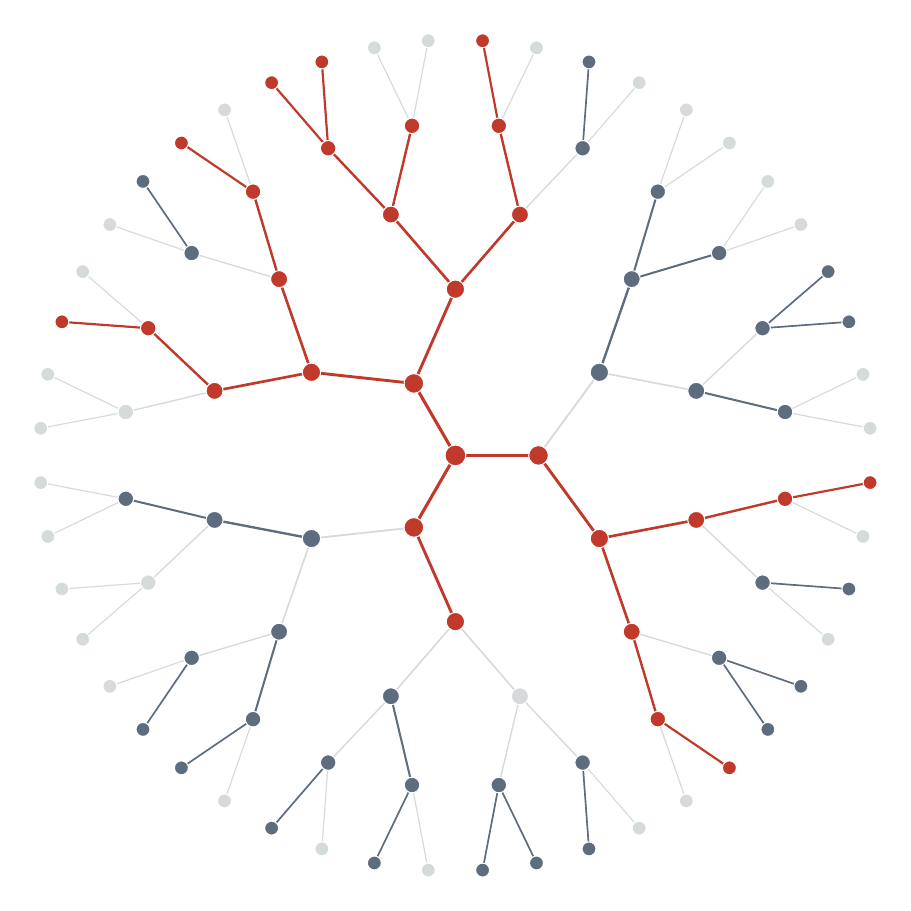}
    \caption{Bethe lattice percolation. Red: infinite cluster. Blue: finite clusters.}
    \label{fig:bethe_lattice}
\end{figure}

In Sec.~\ref{sec:model}, we introduce the percolation data model. In Sec.~\ref{sec:algorithm}, we propose the cyclic coalescent algorithm. In Sec.~\ref{sec:generating_synthetic_data}, we describe a procedure to generate synthetic data following the percolation model. We describe our experiments in Sec.~\ref{sec:experiments} and conclude with discussion in Sec.~\ref{sec:discussion}.

\section{Percolation Model}\label{sec:model}

\subsection{Percolation theory}

The branch of statistical physics concerned with the properties of clusters of randomly occupied sites or bonds on a lattice is called percolation theory \citep{stauffer2018introduction}. In this framework, sites or bonds on a specified lattice are occupied independently at random with probability $p$, forming contiguous clusters.\footnote{Appendix~\ref{appendix:terminology} describes field-specific terms used in this paper.} We consider a hypercubic lattice of dimension $d$ and scale $L$. Percolation's most notable property is a phase transition. Above a critical occupation probability $p=p_c$, the system \textit{percolates}: an infinite cluster emerges that scales with the system size. Below this transition, only finite clusters exist.

Above an upper critical dimension believed to equal $d_u=6$, percolation enters an exactly solvable mean-field regime. A random path in high-dimensional space is vanishingly unlikely to self-intersect, so percolation clusters are cycle-free, and the lattice can be approximated as a tree. Specifically, mean-field percolation is modeled using the Bethe lattice, an infinite tree in which all nodes have equal degree $z$. Fig.~\ref{fig:bethe_lattice} illustrates this with $z=3$. The percolation threshold in this setting is $p_c = 1/(z-1)$. Using the Bethe lattice to approximate a hypercubic lattice gives $z=2d$ and $p_c = 1/(2d-1)$. For mean-field percolation, site and bond percolation are essentially equivalent.

Percolation clusters have interesting statistical and geometric properties \citep{stauffer2018introduction}. Let $s \ge 1$ denote a cluster's size. Close to criticality for $s \gg 1$, mean-field percolation clusters have a size distribution $n_s(p) = s^{-\tau} e^{-cs}$, where $\tau = 5/2$, $c \propto (p_c - p)^{1/\sigma}$, and $\sigma = 1/2$. Exactly at criticality with $p = p_c$, the cluster size distribution becomes a pure power law. Geometrically, percolation clusters are fractal objects with fractal dimension $D < d$. A critical mean-field percolation cluster has $D = 4$ and can be visualized as a tree embedded as a branching random walk on a lattice.

\subsection{Data model}\label{sec:data_model}

Our data model is based on site percolation. Let $\mathcal{X} = [0,1)^d \subset \mathbb{R}^d$ denote the $d$-dimensional data space from which potential data points $\mathbf{x} \in \mathcal{X}$ are drawn. We fix the unit cube without loss of generality. We discretize $\mathcal{X}$ as a hypercubic lattice of linear size $L$ and spacing $1/L$,

\begin{equation}
    \Lambda = \tfrac{1}{L}\mathbb{Z}^d \cap \mathcal{X}.
\end{equation}

Each site of $\Lambda$ is independently in-distribution with probability $p \in [0,1]$, giving a random subset $\mathcal{S} \subseteq \Lambda$ of expected size $\mathbb{E}|\mathcal{S}| = p L^d$. The continuum limit is recovered as $L \to \infty$. An empirical dataset $\mathcal{D} = \{\mathbf{x}_1, \ldots, \mathbf{x}_{|\mathcal{D}|}\}$ is then drawn i.i.d. uniformly from $\mathcal{S}$.

The above model was previously described by \citet{brill2024neural}. In that work, the target function was modeled as a set of random continuous functions, each supported on a separate cluster in $\mathcal{S}$. We now simplify and augment that picture by making three crucial assumptions.

\begin{itemize}
    \item \textbf{Mean-field}: We assume $d \gg d_u$ so that the Bethe lattice model applies. High-dimensional inputs are a realistic assumption for natural data.
    \item \textbf{Criticality}: We set $p = p_c$. This regime may enable both useful and efficient learning \citep{brill2024neural}. Combined with the mean-field assumption, this makes the model analytically tractable.
    \item \textbf{Hierarchy}: We constrain targets to follow a simple hierarchical construction, stated next.
\end{itemize}

We generate target values hierarchically. Let $\mathcal{F}$ be a binary forest of latent variables. Initialize it as the trivial forest with one leaf $z_\mathbf{x}$ associated with each site $\mathbf{x} \in \mathcal{S}$. As a construction device, we activate the bonds between each pair of adjacent sites in $\mathcal{S}$ sequentially in a uniformly random order. However, only a partial order on bonds recorded by $\mathcal{F}$ plays any further role.

The mean-field assumption ensures that every bond connects two distinct clusters $\mathcal{C}_1$ and $\mathcal{C}_2$. For each bond, we introduce a fresh latent, associate it with each
site in $\mathcal{C}_1 \cup \mathcal{C}_2$, and make it the parent of the existing topmost latents of $\mathcal{C}_1$ and $\mathcal{C}_2$. Writing $\mathcal{Z}_\mathbf{x}$
for the top-down path of latents in $\mathcal{F}$ associated with site $\mathbf{x}$, the target at $\mathbf{x}$ is, for some function $f$,

\begin{equation}\label{eq:target_defn}
    y_\mathbf{x} = f(\mathcal{Z}_\mathbf{x}).
\end{equation}

\subsection{Model properties}

The percolation model is quite simple, but it has rich structure matching properties associated with natural data distributions. Mean-field percolation is analytically tractable and well studied, and its statistical and geometric properties are textbook results \citep{stauffer2018introduction}. At criticality, the data consist of sparse clusters power-law-distributed with exponent $5/2$. Furthermore, clusters are low-dimensional fractals with intrinsic dimension $D=4$. \citet{brill2024neural} analyzes the scaling laws that result from the tradeoff between these two competing sources of power-law structure. The graph of each percolation cluster has degree distribution $k \sim 1 + \text{Binomial}(z-1,~p)$. For large $z \sim 2d$ and $p = p_c$, we make the Poisson approximation $k \sim 1 + \text{Poisson}(1)$.

Each cluster forms a functional equivalence class, in the sense that any subset of its inputs shares at least one latent. The construction is statistically self-similar. Decomposing the latent forest $\mathcal{F}$ splits a cluster into identically distributed subclusters. The structure of $\mathcal{F}$ implicitly records information about each latent. In particular, for a given latent, we define its size $s_\mathrm{latent}$ as the number of leaf nodes at or below that latent; its cluster size $s_\mathrm{cluster}$ as the number of leaf nodes  at or below the root of that latent's tree; and its depth $d_\mathrm{latent}$ as the edge length of the path to the latent from its tree's root. The latents in $\mathcal{F}$ serve multiple roles. Because they correspond to bonds in data space, they indicate that merged subclusters share similar input features. For the same reason, they identify split points between subclusters. Finally, they provide latent variables used to compute targets. 

The latent tree $\mathcal{F}$ records hierarchical relations among subsets of inputs, thereby describing taxonomic hierarchical structure. Hierarchical latent variables associated with subclusters may provide a model for hierarchical context features identified in LLMs \citep{gurnee2023finding, brill2025representation}. We discuss taxonomic and compositional hierarchical structure further in Sec.~\ref{sec:discussion}.

\section{Cyclic Coalescent}\label{sec:algorithm}

Directly simulating high-dimensional percolation is intractable. Instead, efficient methods to construct a cluster grow it from a starting site \citep{leath1976cluster} or use invasion percolation \citep{mertens2018percolation}. However, these methods do not explicitly represent self-similar hierarchical structure. Furthermore, our purpose is not to experimentally probe percolation itself, but to apply the critical mean-field regime as a data model. To that end, we leverage equivalences between percolation and several other mathematical constructions, using these mappings to develop an algorithm that efficiently generates clusters with a hierarchical decomposition equivalent to the model in Sec.~\ref{sec:data_model}.

A mean-field percolation cluster can be thought of as a size-conditioned Galton-Watson branching process, and mapped to a Brownian excursion \citep{roch_mdp_2024, aldous1993continuum, font2016percolation}. These mappings imply that critical mean-field percolation clusters have the same distribution as uniform random labeled trees \citep{aldous1991continuum1, aldous1991continuum2}, which are well-studied combinatorial objects \citep{moon1970counting}. In particular, a deep connection exists between random trees and a merger process known as additive coalescence \citep{aldous1998standard, pitman1999coalescent, Aldous1999coalescence}.

\begin{figure*}[tb]
    \centering
    \includegraphics[width=0.6\linewidth]{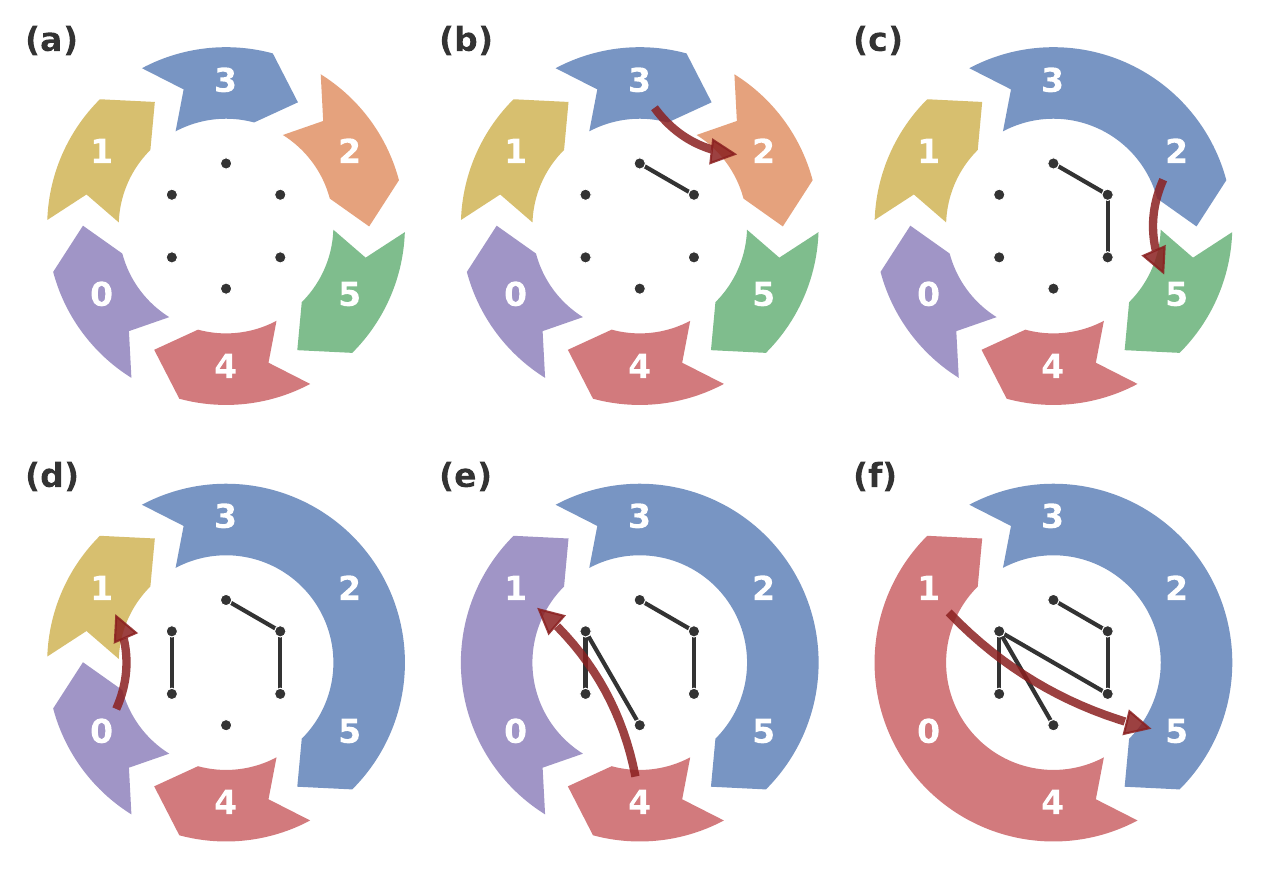}
    \caption{Example run of the cyclic coalescent. (a) Initialize nodes in a random cycle. (b--f) At each step, connect a random node to a random node in its block's successor, merging those blocks.}
    \label{fig:cyclic_coalescent}
\end{figure*}

In an elegant construction, \citet{pitman1999coalescent} showed that the time-reversal of deleting a random tree's edges uniformly at random is equivalent to merging a forest of smaller trees in an additive coalescent process. At each step, two randomly chosen trees merge with a probability proportional to their total size. Our method draws heavily on \citet{pitman1999coalescent}, and we review essential results from that work in detail in Appendix~\ref{sec:theorems_prereqs}. In each step of Pitman's construction, out of $k$ remaining trees with total size $n$, a pair of trees $T_i$ and $T_j$ is picked to merge with probability

\begin{equation}
    P(i, j) = \frac{n_i + n_j}{n(k-1)},
\end{equation}

where $n_i$ and $n_j$ are those tree's sizes. Next, a node is picked independently and uniformly at random from each of $T_i$ and $T_j$, and an edge added between those nodes. Amazingly, Theorem~\ref{thm:pitman_theorem5} states that this process is exactly equivalent to uniform random edge insertion in a random tree. What's more, it is equivalent to the hierarchical procedure described in Sec.~\ref{sec:data_model}, allowing us to repurpose coalescence as a procedure to build up a hierarchy of latent variables.

We propose a variant of Pitman's construction called the \textit{cyclic coalescent}. As with the standard coalescent construction, the cyclic coalescent uniformly samples random trees. Furthermore, it admits a simple and highly efficient algorithmic implementation that runs in almost linear $O(n~\alpha(n))$ time. Our construction is based on two key ideas. First, picking a node uniformly at random plays double duty. One random draw simultaneously picks a tree with a probability that depends on its size, and picks a node from it uniformly at random. Second, we arrange the trees in a fixed random cyclic order. Rather than allowing unrestricted merges among all trees, the chosen tree simply merges with its successor. The initial random cyclic order ensures the distribution's uniformity. Fig.~\ref{fig:cyclic_coalescent} illustrates the cyclic coalescent's operation. We state the construction formally and prove its correctness in Appendix~\ref{sec:theorems_cyclic}.

We implement the cyclic coalescent using a union-find data structure \citep{tarjan1975efficiency}.  Pseudocode is shown in Algorithm~\ref{alg:cyclic-coalescent}. A union-find algorithm efficiently maintains a partition of subsets, supporting a \texttt{find} operation to identify an element's subset and a \texttt{union} operation to merge two subsets. Union-find algorithms have wide application for graph problems, including for studying site and bond percolation \citep{newman2001fast}. Our implementation uses union by rank and path halving \citep{tarjan1984worst}. Each union-find iteration takes amortized $O(\alpha(n))$ time, where $\alpha$ is the extremely slow-growing inverse Ackermann function and is effectively constant for physically realizable $n$. Initialization is $O(n)$ and the algorithm runs for $n - 1$ iterations, giving $O(n~\alpha(n))$ total time complexity.

\begin{algorithm}[tb]
   \caption{\emph{Cyclic coalescent}: $O(n\,\alpha(n))$ uniform tree sampler with coalescent decomposition}
   \label{alg:cyclic-coalescent}
\begin{algorithmic}
   \STATE {\bfseries Input:} Number of nodes $n$
   \STATE {\bfseries Output:} Uniform random labeled tree $G$ on $\{1,\ldots,n\}$ and its rooted binary merger tree $B$
   \STATE {\bfseries State:} Blocks form a cyclic order maintained by union-find supporting block lookup, $B$-root lookup, cyclic-successor query, and merge-with-successor in amortized $O(\alpha(n))$ time.
   \STATE Arrange $1,\ldots,n$ in a uniformly random cyclic order, with each label its own block.
   \STATE Initialize $B$ as the forest of leaves $1,\ldots,n$.
   \FOR{$k=1$ {\bfseries to} $n-1$}
   \STATE Draw $u \sim \mathrm{Unif}(\{1,\ldots,n\})$; let $A$ be the block of $u$ and $A'$ its cyclic successor.
   \STATE Draw $v \sim \mathrm{Unif}(A')$ independently, and add edge $\{u,v\}$ to $G$.
   \STATE Let $r_A, r_{A'}$ be the $B$-roots of $A, A'$; create a new internal node in $B$ with children $r_A, r_{A'}$.
   \STATE Merge $A$ into $A'$, and assign the new internal node as the $B$-root of the merged block.
   \ENDFOR
   \STATE {\bfseries return} $G,\ B$
\end{algorithmic}
\end{algorithm}

\section{Synthetic Dataset}\label{sec:generating_synthetic_data}

\textbf{Cluster distribution.}\quad We use an iterative preferential-attachment procedure to generate clusters with a power-law size distribution. Each step increments the dataset size by one. At each step, either a new cluster is created with probability $p_{01}$, or else an existing cluster is chosen with probability proportional to its size $s$. That cluster's size is then increased to $s + 1$. To obtain a distribution consistent with critical mean-field percolation, we set $p_{01} = 1/3$, which we derive in Appendix~\ref{appendix:cluster_statistics}. This iterative procedure induces a self-consistent size-independent ordering among data points for a specified random seed, although we do not rely on this property for our experiments in this work.

\textbf{Cluster geometry.}\quad For each cluster, we use the cyclic coalescent to generate a random tree and an associated hierarchical latent decomposition. We do this independently for each cluster in the dataset. This yields a forest of data points and a corresponding forest of rooted binary latent trees.

\textbf{Embeddings.}\quad After generating the graph structures of all clusters, the nodes are embedded as data points in a $d$-dimensional vector space. The dimension $d$ is a hyperparameter that can be set arbitrarily. Each cluster is embedded by first randomly choosing a root node and then iteratively embedding each node's neighbors following a branching random walk. The random choice of the root node and the random generation of the step direction at each node when embedding the graph are both performed hierarchically, following the cluster's latent tree. This embedding procedure yields clusters that have consistent geometrical structure regardless of the number of iterations used in the generation process.

\textbf{Targets.}\quad We specialize Eq.~\ref{eq:target_defn} to a regression task that is a linear function of the latent variables. Specifically, we assign each latent variable $z_i \in \mathcal{F}$ a value $z_i \sim \mathcal{N}(0, 1)$. We record each latent's value and depth. For each point $\mathbf{x}$, we compute its target as the normalized sum of its associated latents' values,

\begin{equation}
        y_\mathbf{x} = \frac{1}{\sqrt{1 + d_\mathbf{x}}} \sum_{i=0}^{d_\mathbf{x}} z_i^{(\mathbf{x})},
    \label{eq:leaf_value}
\end{equation}

where $z_i^{(\mathbf{x})}$ is the value of the latent at depth $i$ and $d_\mathbf{x}$ is the depth of point $\mathbf{x}$.

Each latent has an expected contribution to the mean-squared error (MSE) proportional to its  number of associated data points, given by its size $s_\mathrm{latent}$, divided by the expected number of latents per point. Theorem~\ref{thm:pitman_theorem5} states a bijection between the time-reversal of edge deletion in a random tree and an additive coalescent process, which implies that the number of merges a random node undergoes equals the number of cuts needed to isolate a node in a random tree. This has expectation $\sqrt{\pi s_\mathrm{cluster} / 2}$ \citep{moon1970counting}. We therefore stratify latents by the figure of merit $\mathrm{FOM} \equiv s_\mathrm{latent} / \sqrt{s_\mathrm{cluster}}$.

\section{Experiments}\label{sec:experiments}

\textbf{Data generation.}\quad We generated two datasets as described in Section~\ref{sec:generating_synthetic_data}. The \textit{one-cluster} dataset consists of a single percolation cluster with $2\times10^5$ data points. This dataset allows us to study the cluster properties in isolation. Second, a \textit{multi-cluster} dataset with $2\times10^6$ data points was generated using the full percolation model. The multi-cluster dataset was filtered to drop clusters with fewer than 500 data points, to ensure that each cluster had sufficient data points for training and validation. Both datasets were stored as input embeddings $\mathbf{X} \in \mathbb{R}^{n \times d}$ and scalar labels $\mathbf{y} \in \mathbb{R}^n$, along with the ground-truth latent features. Full dataset hyperparameters are given in Table~\ref{tab:data_hyperparameters} in Appendix~\ref{appendix:hyperparameters_data}.

\begin{figure*}[htbp]
    \centering
    \begin{subfigure}[b]{0.45\textwidth}
        \centering
        \includegraphics[trim={0.2cm 0.3cm 2cm 2cm}, clip, width=\textwidth]{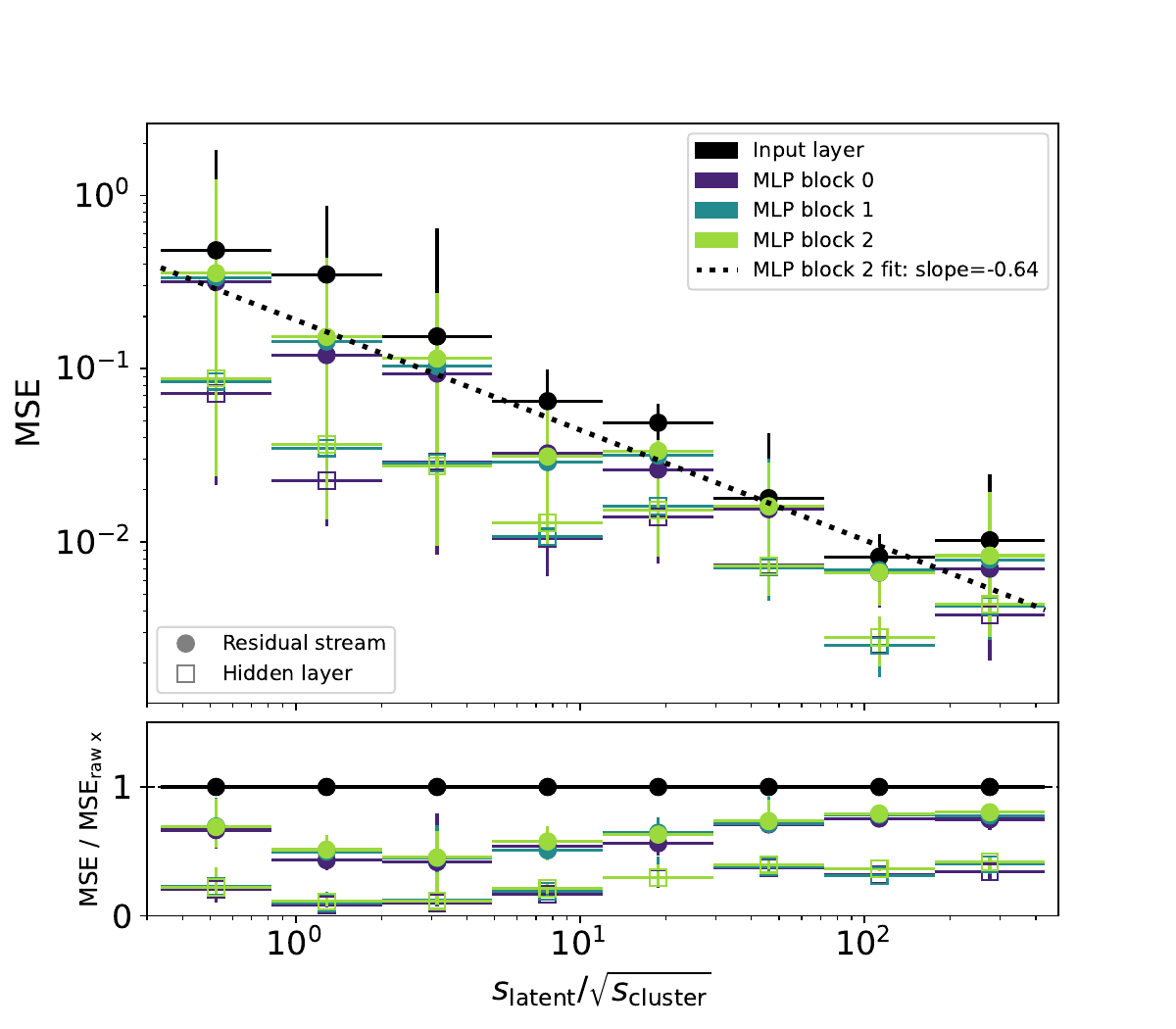}
        \caption{}
        \label{subfig:1clusterMSE_zu}
    \end{subfigure}
    \hfill
    \begin{subfigure}[b]{0.45\textwidth}
        \centering
        \includegraphics[trim={0.2cm 0.3cm 2cm 2cm}, clip, width=\textwidth]{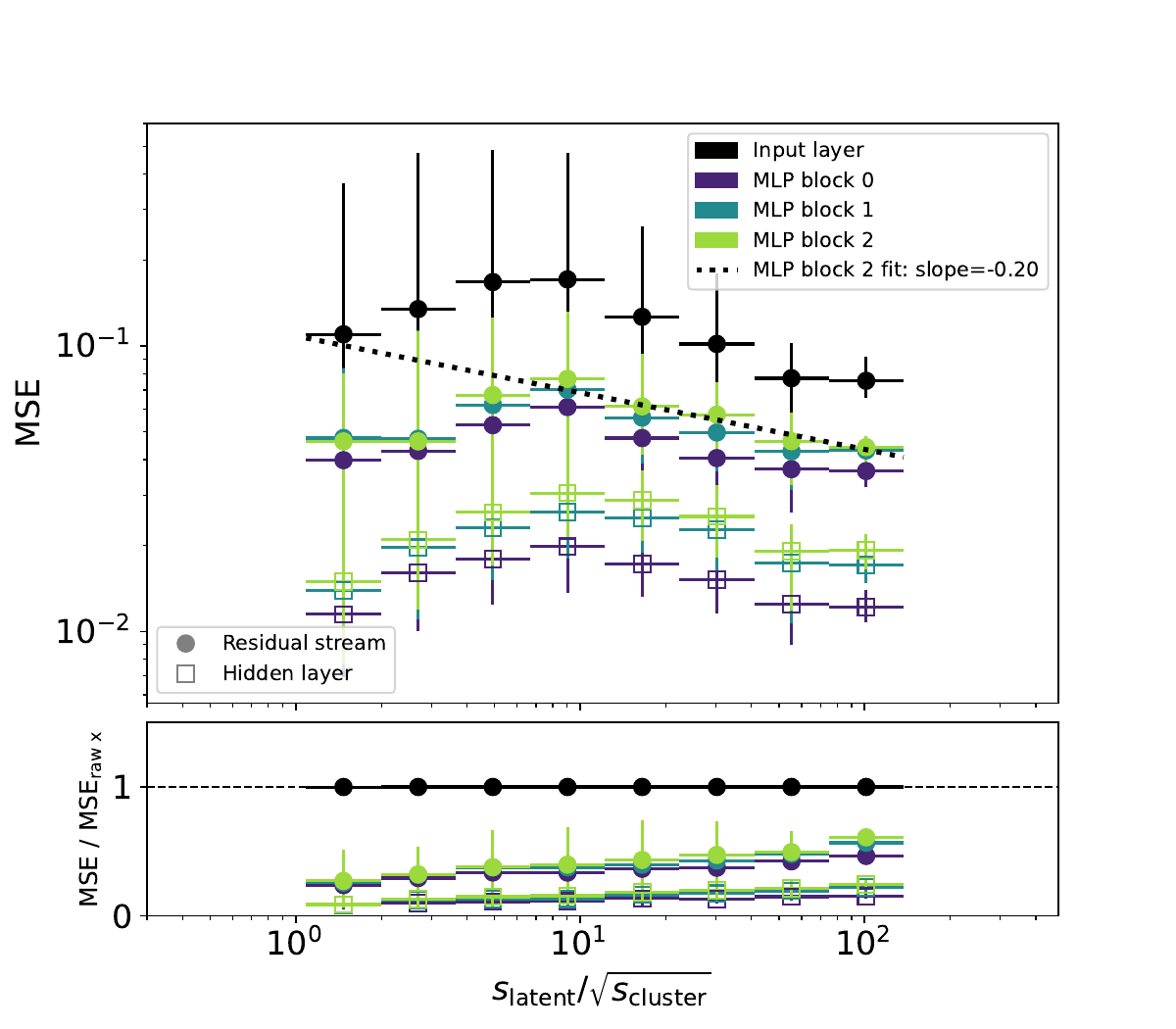}
        \caption{}
        \label{subfig:multiclusterMSE_zu}
    \end{subfigure}
    \caption{Per-latent linear probe performance for individual latent values $z_i$, showing (a) the one-cluster dataset and (b) the multi-cluster dataset. Dots (squares) mark the median MSE for the residual stream (hidden activations). Error bars show 25th and 75th percentiles. The ratio panels show each layer's per-latent MSE normalized by the MSE of probes trained on the raw input.}
    \label{fig:zu_MSE_probes}
\end{figure*}

\textbf{Neural network training.}\quad We trained a residual multilayer perceptron (MLP) on the percolation data. This architecture was chosen to resemble a transformer feed-forward block. Our MLP consists of a linear input projection, $L$ MLP blocks, and a linear output layer. Each MLP block expands by a factor of 4, applies a ReLU activation, and projects back to $d_\text{model}$ with a residual connection,

\begin{equation}
    \mathbf{x}' = \mathbf{x} + W_2 \sigma(W_1 \mathbf{x} + \mathbf{b}_1) + \mathbf{b}_2,
    \label{eq:mlp_block}
\end{equation}

where $\mathbf{x}$ and $\mathbf{x}'$ are the MLP block's input and output, $W_1$ and $W_2$ are weight matrices, $\mathbf{b}_1$ and $\mathbf{b}_2$ are bias vectors, and $\sigma$ denotes the ReLU activation. For the one-cluster dataset we used $d_\text{model}=256$ and $L=3$ blocks, and for the multi-cluster dataset we used $d_\text{model}=512$ and $L=3$ blocks. Both models were trained using the AdamW optimizer \citep{loshchilov2017decoupled} and a cosine learning rate decay schedule. Full hyperparameters are provided in Table~\ref{tab:mlp_hyperparameters} in Appendix \ref{appendix:hyperparameters_model}.

\textbf{Model performance.}\quad We report MLP performance against two simple baselines, shown in Table~\ref{tab:performance}. Ridge regression performs significantly worse, confirming the task's nonlinearity. We also compute 1-nearest-neighbor (1NN) regression using the full dataset to estimate the best practically achievable performance. The MLP performs comparably to this best-case estimate. 

\begin{table}
\centering
\caption{Model performance.}
\label{tab:performance}
\begin{tabular}{clcc}
\hline
\textbf{Metric} & \textbf{Model} & \textbf{One-cluster} & \textbf{Multi-cluster} \\
\hline
\multirow{3}{*}{$R^2$} & Ridge & 0.51\phantom{0}  & 0.39 \\
                       & 1NN   & 0.94\phantom{0}  & 0.86 \\
                       & MLP   & 0.94\phantom{0}  & 0.88 \\
\hline
\multirow{3}{*}{MSE}   & Ridge & 0.20\phantom{0}  & 0.59 \\
                       & 1NN   & 0.024            & 0.13 \\
                       & MLP   & 0.026            & 0.12 \\
\hline
\end{tabular}
\end{table}

\textbf{Linear probes.}\quad We trained linear probes \citep{alain2016understanding} to assess if MLPs trained to solve the percolation task linearly represent the ground-truth latent variables $z_i$ defined in Eq.~\ref{eq:leaf_value}. Separate probes were trained to regress the values of all latents at each given depth, using all training points of larger depth. Stratifying by depth fixes a unique latent for each point. We excluded latents with $s_\mathrm{latent} < 150$ and their corresponding points to have enough samples for training and validation.

Linear probes were fit from the activations to the latent value $z_i$. We also fit probes to the partial cumulative sum descending the latent path to $z_i$, with similar results, shown in Appendix~\ref{appendix:additional_experiments}. Probes were computed for depths ${0, 5, 10, \ldots, 200}$ on the raw input and at each network layer. Probe performance was evaluated using the per-latent MSE, computed separately on the validation set for each latent.

\begin{figure}
    \centering
    \includegraphics[trim={0.2cm 0.2cm 0.2cm 0.2cm}, clip, width=0.43\textwidth]{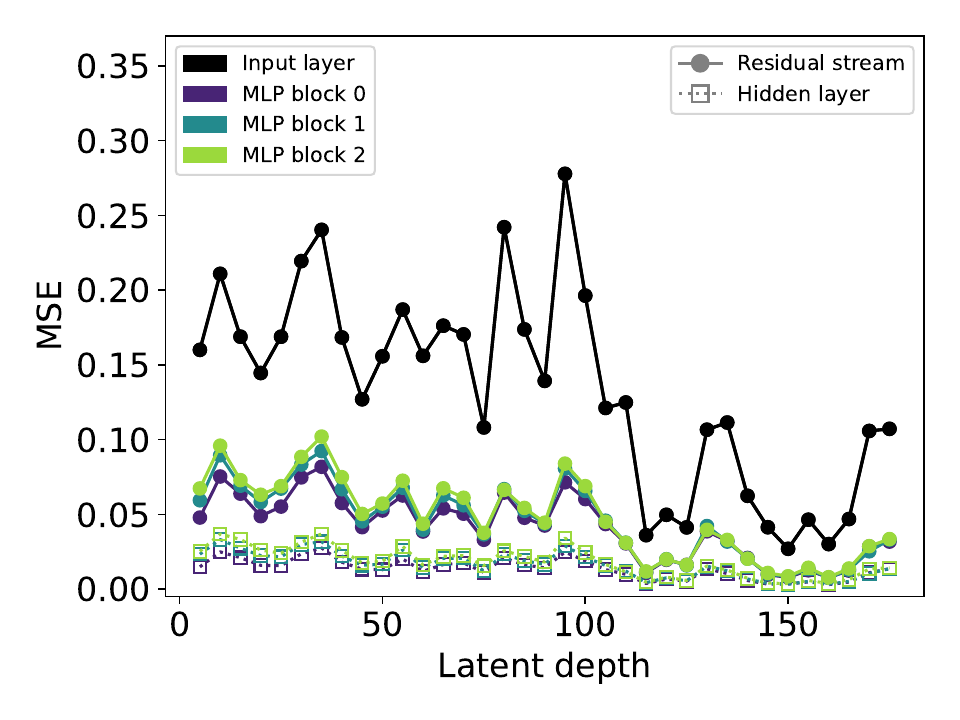}
    \caption{Linear probe recovery of latent values $z_{i}$ for the one-cluster dataset, stratified by latent depth.}
    \label{fig:globalmse}
\end{figure}

\section{Discussion}\label{sec:discussion}

\textbf{Linear representations.}\quad Fig.~\ref{fig:zu_MSE_probes} shows the probe results. The ground-truth latent variables can be linearly decoded from the network's activations. Per-latent probes trained on MLP activations have consistently lower MSE than on the raw input. The hidden layer activations have lower probe MSE than the residual stream. The ratio plots show that the probe performance gap compared to raw input decreases with increasing FOM. We hypothesize that latents representing coarser subtrees may be easier to decode from the input geometry, requiring less nonlinear computation.

\textbf{Power-law trends.}\quad The per-latent MSE exhibits an apparent power-law trend as a function of FOM. We visualize this with power-law fits in Fig.~\ref{fig:zu_MSE_probes}. In Fig.~\ref{subfig:multiclusterMSE_zu}, the MSE drop at low FOM is an artifact of the minimum cluster size filter, which removes the smallest clusters and suppresses the low-FOM bins. Fig.~\ref{fig:globalmse} shows the probe MSE as a function of latent depth. No power-law trend is visible, suggesting that FOM captures the latents' statistical importance more effectively than depth alone. For depths above 100, the global MSE decreases due to a selection bias, as only the few largest clusters have the deepest latents.

\textbf{No layerwise progression.}\quad The probe MSE is very similar for all MLP blocks, with the possible exception that the MSE appears sightly lower for MLP block 0 for multi-cluster data in Fig.~\ref{subfig:multiclusterMSE_zu}. Furthermore, no clear pattern is observed for any FOM bin across layers. This lack of trend suggests that the network may not represent latents following an interpretable organization across layers, such as building up a hierarchical sum. This hypothesis is supported by the lack of layerwise trend in the cumulative latent sums, shown in Fig.~\ref{fig:sum_zu_MSE_probes} in Appendix~\ref{appendix:additional_experiments}.

\textbf{Mechanistic interpretability.}\quad This work develops a synthetic data model meant to complement empirical studies interpreting LLMs. Because the percolation model has ground-truth hierarchical latent structure, future work could test its validity as a useful data model by investigating whether it causes analogs of SAE pathologies such as feature splitting and absorption \citep{bricken2023monosemanticity, chanin2024absorption}. In addition, the relationship between data structure and structure in neural activations is poorly understood. A fruitful avenue of further research may be to investigate whether percolation data produces activation structure with features similar to synthetic models \citep{chanin2026synthsaebench}. We discuss additional related work in Appendix~\ref{appendix:related_work}.

\textbf{Hierarchy.}\quad The percolation model's latent forest $\mathcal{F}$ describes a taxonomic hierarchy of set inclusion relations among samples. This is a hierarchy distinct from, and complementary to, composition among features, which has been described using synthetic data models \citep[e.g.][]{cagnetta2024deep}. Table~\ref{tab:hierarchy} and Fig.~\ref{fig:hierarchy_diagram} compare taxonomic and compositional hierarchical relations. In addition, the percolation model naturally maintains an explicit instance/class distinction within a taxonomic hierarchy \citep{brachman1983and}. The latents represent classes and data points are instances.

\begin{figure}
    \centering
    \includegraphics[trim={2.0cm 0.0cm 2.0cm 0.0cm}, clip, width=0.95\linewidth]{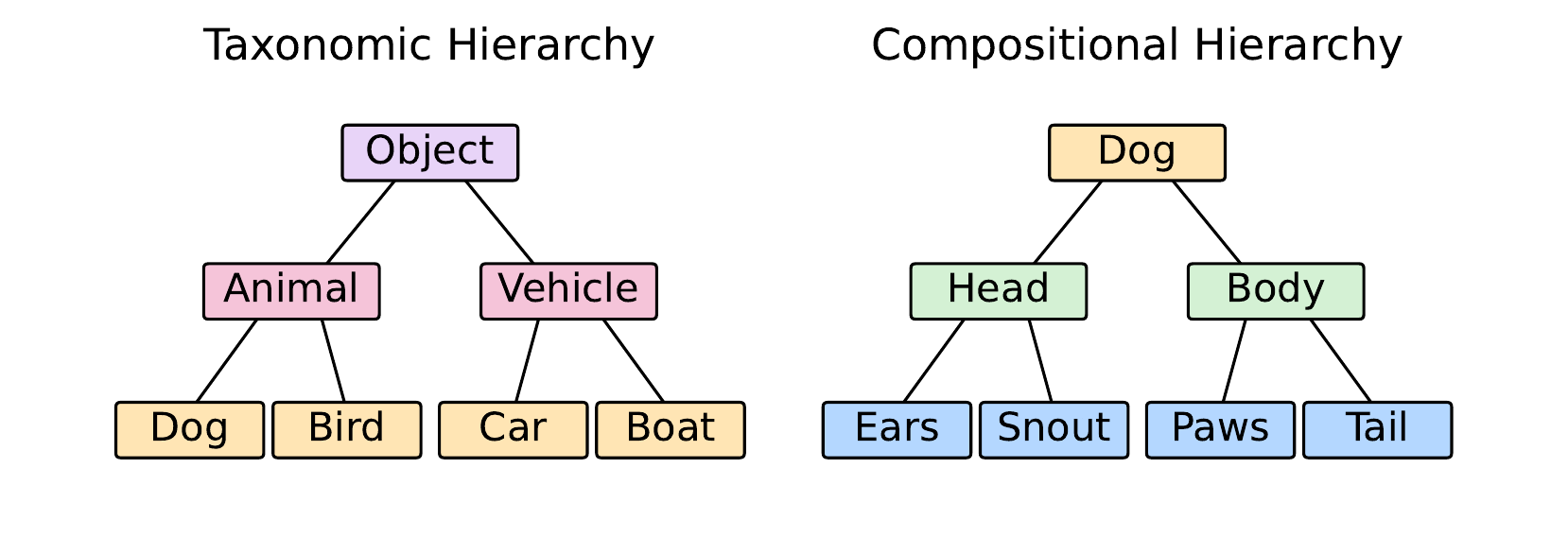}
    \caption{A comparison of taxonomic and compositional hierarchical structures.}
    \label{fig:hierarchy_diagram}
\end{figure}

\begin{table}
\centering
\caption{Hierarchical relations.}
\label{tab:hierarchy}
\begin{tabular}{ccc}
\hline
& \textbf{Taxonomic} & \textbf{Compositional} \\
\hline
Unit & sample & feature\\
Description & set inclusion & part-whole\\
Relation & IS-A & HAS-A\\
Semantics & hyponymy & meronymy\\
\hline
\end{tabular}
\end{table}

\textbf{Limitations.}\quad The correlational probing experiments in this work serve as a proof of concept for investigating the percolation model as an interpretability testbed. Causal interventions are needed to validate specific interpretability hypotheses \citep{meng2022locating, chan2022causal}.

The percolation model lacks many features of real data. It assumes that the data distribution is populated at random. It does not model compositional relationships among input features. The synthetic dataset we present in this work supports vectorized data with scalar regression labels only. Extending the model to support tokenized data and classification tasks is a direction for future work.

Finally, we present a critical mean-field percolation model. One possible avenue for generalization is to consider percolation beyond the critical regime. In particular, supercritical percolation allowing for clusters with cycles may provide a framework to model overlapping compositional concepts.

\section*{Acknowledgements}

We thank Lauren Greenspan, Jennifer Lin, Andrew Mack, Nischal Mainali, Lucas Teixeira, and Dmitry Vaintrob for helpful discussions and feedback on this work. AB is supported by grants from Coefficient Giving and from the Long-Term Future Fund (EA Funds).

\bibliography{main}
\bibliographystyle{icml2026}

\newpage
\appendix
\onecolumn
\section{Terminology}\label{appendix:terminology}

Because this work ties together methods from several fields, we use field-specific terms in various sections to refer to the same objects. Table~\ref{tab:terminology} translates among them.

\begin{table}[h!]
\centering
\caption{Terminology correspondences across fields.}
\begin{tabular}{lll}
\toprule
\textbf{Field} & & \\
\midrule
Percolation    & site       & bond   \\
Graphs \& trees  & node       & edge   \\
\midrule
Data model     & data point & latent \\
\bottomrule
\end{tabular}
\label{tab:terminology}
\end{table}

\section{Model training}\label{appendix:model_training}

\subsection{Data hyperparameters}\label{appendix:hyperparameters_data}

Table~\ref{tab:data_hyperparameters} lists the hyperparameters used to generate the one-cluster and multi-cluster datasets described in Sec.~\ref{sec:experiments}. For the multi-cluster dataset, a minimum cluster size filter is applied as a post-processing step after generation, where cluster size refers to the number of points per cluster.

\begin{table}[H]
\centering
\caption{Data generation hyperparameters.}
\label{tab:data_hyperparameters}
\begin{tabular}{lcc}
\hline
\textbf{Hyperparameter} & \textbf{One-cluster} & \textbf{Multi-cluster} \\
\hline
Mode & \texttt{one\_cluster} & \texttt{distribution} \\
$n$ (dataset size) & $2\times10^5$ & $2\times10^6$ \\
$d$ (embedding dimension) & 100 & 100 \\
Graph seed & 0 & 0 \\
Embedding seed & 10{,}000 & 10{,}000 \\
Value seed & 20{,}000 & 20{,}000 \\
\hline
\multicolumn{3}{l}{\textit{Post-processing}} \\
Min.\ cluster size filter & --- & 500 \\
\hline
\end{tabular}
\end{table}

\subsection{Model hyperparameters}\label{appendix:hyperparameters_model}

Table~\ref{tab:mlp_hyperparameters} lists the full set of hyperparameters for the residual MLP models trained in Section~\ref{sec:experiments}. Both models are implemented in PyTorch and share the same architecture and optimisation scheme, differing only in model width and number of training epochs. Data is split 80/10/10 into training, validation, and test sets.

\begin{table}[H]
\centering
\caption{Neural network training hyperparameters.}
\label{tab:mlp_hyperparameters}
\begin{tabular}{lcc}
\hline
\textbf{Hyperparameter} & \textbf{One-cluster} & \textbf{Multi-cluster} \\
\hline
$d_\text{model}$ & 256 & 512 \\
$N$ (MLP blocks) & 3 & 3 \\
Expansion factor & 4 & 4 \\
Activation & ReLU & ReLU \\
Input dimension & 100 & 100 \\
Output dimension & 1 & 1 \\
\hline
Batch size & 1024 & 1024 \\
Epochs & 500 & 2500 \\
Loss & MSE & MSE \\
Optimizer & AdamW & AdamW \\
Learning rate & $10^{-4}$ & $10^{-4}$ \\
Weight decay & $0.01$ & $0.01$ \\
Scheduler & Cosine & Cosine \\
$T_\text{max}$ & 500 & 2500 \\
$\eta_\text{min}$ & $10^{-6}$ & $10^{-6}$ \\
\hline
Random seed & 42 & 42 \\
\end{tabular}
\end{table}

\section{Additional related work}\label{appendix:related_work}

\textbf{Fractals in machine learning}\quad Fractals generated using
iterated function systems are used  to create synthetic training data for image models \citep{nakamura2024scaling, baradad2021learning, kataoka2020pre}. \citet{bloem2017expectation} present a expectation-maximization algorithm for fitting a fractal model to data. \citet{malach2019deeper} study the expressivity properties of deep neural networks trained on fractal data distributions, finding that deep but not shallow networks can efficiently express these distributions. Fractal-based methods are used to estimate the intrinsic dimension of data \citep{grassberger1983measuring, karbauskaite2016fractal}. \citet{alabdulmohsin2024fractal, alabdulmohsin2025tale} argue that natural language viewed from an information-theoretic time-series perspective exhibits fractal patterns, which may be helpful for detecting LLM-generated texts.

\textbf{Synthetic data models}\quad \citet{nanda2023progress} apply modular addition as a toy model to investigate progress measures for grokking via mechanistic interpretability. \citet{allen2023physicsknowledge} introduce a synthetic dataset of uniformly distributed discrete knowledge pieces to measure knowledge storage and extraction in transformers. \citet{pan2025understanding} consider a variant of this model with power-law-distributed knowledge pieces. \citet{liu2022transformers} find theoretically and empirically that transformers can simulate finite-state automata using non-recurrent hierarchical shortcut solutions. \citet{brinkmann2024mechanistic} perform a mechanistic analysis of a transformer trained to perform a symbolic multi-step reasoning task, finding that does so using parallelized depth-bounded recurrent mechanisms.

\textbf{Percolation in deep learning}\quad \citet{lubana2024percolation} apply percolation on a bipartite graph to propose a model of emergent capabilities in transformers. \citet{liang2024diffusion} identify a connection between manifold learning in diffusion models
and continuum percolation. \citet{devlin2025dropout} analyze dropout neural network training by considering percolation among the network of weights.

\textbf{AI interpretability \& alignment}\quad The latent hierarchy proposed by the percolation model may complement research directions that aim to develop a theory of concepts \cite{wentworth2025natural, eisenstat2condensation}. A data distribution compatible with the percolation model, consisting of a heavy-tailed distribution of sparse low-dimensional clusters, can be used to study scaling laws of general intelligence \citep{brill2025model}. Developing a better understanding of the structure of data is a key element in developing a scale-aware basis for mechanistic interpretability \citep{greenspan2026towards}, advancing AI alignment \citep{lehalleur2025you}, and building towards a scientific theory of deep learning \citep{simon2026there}.

\section{Cluster statistics}\label{appendix:cluster_statistics}

In this section, we derive the cluster creation probability $p_{01}$ that, in a preferential attachment process, yields a power-law distribution consistent with critical mean-field percolation. Our derivation closely follows a master-equation method associated with preferential-attachment models in network science \citep{price1976general, barabasi1999emergence, newman2018networks}.

Let $s \ge 1$ denote a cluster's size. We will call a cluster of size $s$ an $s$-cluster. Let $n_s$ be the cluster probability mass function of $s$. For critical mean-field percolation, $n_s \sim s^{-5/2}$ for large $s$. Note that because each $s$-cluster contains $s$ sites, the probability mass function of $s$ obtained by randomly selecting sites rather than clusters is proportional to $sn_s$ rather than $n_s$.

In the following, we derive the sizes used to generate graphs that are not embedded in a lattice, and so refer to $s$ as the count of nodes (rather than sites).

Let $N$ denote the total number of active nodes that have been generated (i.e., after $N$ iterations), and write $n_s(N)$ to denote the cluster size distribution when there are $N$ active nodes.

We first consider the case $s > 1$. The number of nodes belonging to $s$-clusters after $N$ iterations is

\begin{equation}
    N s n_s(N).
\end{equation}

Adding a node to an $(s - 1)$-cluster increases the number of $s$-clusters by 1. This corresponds to increasing the number of nodes belonging to $s$-clusters by $s$. At each iteration, a node is added to an existing cluster with probability $1 - p_{01}$. If a node is added to an existing cluster, the probability that it is a node belonging to an $(s-1)$-cluster is $(s - 1)n_{s-1}(N)$. All told, the expected number of nodes belonging to $s$ clusters added at each iteration is

\begin{equation}
    (1 - p_{01})s(s - 1)n_{s-1}(N).
\end{equation}

Similarly, adding a node to an $s$-cluster decreases the number of $s$-clusters by 1, since that cluster would become an $(s+1)$-cluster. This corresponds to decreasing the number of nodes belonging to $s$-clusters by $s$. The expected number of nodes belonging to $s$-clusters removed at each iteration is therefore

\begin{equation}
    (1 - p_{01})s^2n_s(N).
\end{equation}

Putting everything together, we have the master equation for $s > 1$,

\begin{equation}\label{eq:master_eq_sgtr1}
    (N + 1) s n_s(N + 1) = N s n_s(N) + (1 - p_{01})s(s - 1)n_{s-1}(N) - (1 - p_{01})s^2n_s(N).
\end{equation}

We consider the limit of a large number of iterations, letting $N \to \infty$. Solving Eq.~\ref{eq:master_eq_sgtr1} yields the recurrence relation,

\begin{equation}\label{eq:master_eqn_sgtr1_soln}
    n_s = \frac{s -1}{s - 1 + y} n_{s-1},
\end{equation}

where $y = (2 - p_{01})/(1 - p_{01})$.

Next, we consider the case $s = 1$. The number of 1-clusters increases by 1 if a new cluster is created, which occurs with probability $p_{01}$. The number of 1-clusters decreases by 1 if a node is added to a 1-cluster, which occurs with probability $(1 - p_{01})n_1(N)$. Since each 1-cluster contains 1 node, the equation for $s=1$ is

\begin{equation}\label{eq:master_eq_seq1}
    (N + 1) n_1(N + 1) = N n_1(N) + p_{01} - (1 - p_{01})n_1(N).
\end{equation}

Solving Eq.~\ref{eq:master_eq_seq1} in the limit $N \to \infty$ gives

\begin{equation}\label{eq:master_eqn_seq1_soln}
    n_1 = \frac{p_{01}}{y(1 - p_{01})},
\end{equation}

Eq.~\ref{eq:master_eqn_sgtr1_soln} and Eq.~\ref{eq:master_eqn_seq1_soln} can be combined to give an expression for $n_s$ for general $s$. We can write this expression in terms of Gamma functions and simplify it using the beta function,

\begin{equation}
    n_s = \frac{\Gamma(s)\Gamma(y)}{\Gamma(s + y)}\frac{p_{01}}{1 - p_{01}} = B(s, y)\frac{p_{01}}{1 - p_{01}}.
\end{equation}

For large $s$, $ B(s, y) \approx s^{-y} \Gamma(y)$, giving

\begin{equation}\label{eq:ns_p01}
   n_s = \frac{p_{01}}{1 - p_{01}} \Gamma\left(\frac{2 - p_{01}}{1 - p_{01}}\right) s^{-(2 - p_{01})/(1 - p_{01})}.
\end{equation}

Equating the exponent to 5/2 gives $p_{01} = 1/3$. Substituting into Eq.~\ref{eq:ns_p01} and integrating, we obtain the complementary cumulative distribution function (CCDF),

\begin{equation}\label{eq:ccdf}
    \mathrm{CCDF} = \frac{\sqrt{\pi}}{4} s^{-3/2}.
\end{equation}

\section{Cyclic coalescent}\label{appendix:cyclic_coalescent}

The cyclic coalescent construction is underpinned by a close relation between random trees and coalescence. As shown by \citet{pitman1999coalescent}, constructing a random (labeled) tree by populating its edges one by one uniformly at random is equivalent to iteratively merging the tree components of a randomly distributed forest in a merger process known as additive coalescence. In this process, tree components merge with a rate proportional to the sum of their sizes. To make this section self-contained, we begin by restating relevant theorems from \citet{pitman1999coalescent}. These are given below with adapted notation as Theorems~\ref{thm:pitman_lemma1},~\ref{thm:pitman_lemma3}, and \ref{thm:pitman_theorem5}. An expository presentation of Theorem~\ref{thm:pitman_lemma1} can be found in \citet{aigner2014proofs}.

Consider the node set $\{1, \dots, n\}$. Identify a \textit{tree} over $\{1, \dots, n\}$ by a set of $n - 1$ edges between nodes, such that every node is connected by at least one edge. Note that this implies that a tree cannot have any cycles or self-edges. Let $\mathcal{T}_n$ be the set of all trees over $\{1, \dots, n\}$, and denote $|\mathcal{T}_n| = T_n$. Call a tree together with the choice of a particular root node a \textit{rooted tree}. Since the root could be any of $n$ nodes, there are $n T_n$ rooted trees over $\{1, \dots, n\}$.

A \textit{forest} over $\{1, \dots, n\}$ is a graph in which each connected component is a tree. Let $\mathcal{F}_{n,k}$ be the set of all forests over $\{1, \dots, n\}$ consisting of $k$ trees. Call a forest together with the choice of a root in each component tree a \textit{rooted forest}. Let $\mathcal{R}_{n,k}$ be the set of all rooted forests over $\{1, \dots, n\}$ consisting of $k$ rooted trees.

\subsection{Prerequisites}\label{sec:theorems_prereqs}

\begin{theorem}[{\citep[][1]{pitman1999coalescent}}]\label{thm:pitman_lemma1}
    For each rooted forest $r_k \in \mathcal{R}_{n,k}$, the number of rooted trees that contain $r_k$ is $N(r_k) = n^{k-1}$.
\end{theorem}

\begin{proof}
    We regard a rooted forest as a directed graph with the edges directed away from the roots. Say that a rooted forest $r$ \textit{contains} another rooted forest $r'$ if the directed graph of $r$ contains the directed graph of $r'$. Call a sequence of rooted forests $r_1, \dots r_k$ a \textit{refining sequence} if $r_i \in \mathcal{R}_{n,k}$ and $r_i$ contains $r_{i+1}$ for all $i$. Let $r_k \in \mathcal{R}_{n,k}$ be a fixed rooted forest. Denote by $N(r_k)$ the number of rooted trees containing $r_k$. Denote by $N^*(r_k)$ the number of refining sequences ending in $r_k$.

    First, we count $N^*(r_k)$ starting from a rooted tree and deleting edges. Suppose that $r_1 \in \mathcal{R}_{n,1}$ contains $r_k$. Then $r_1$ has $k-1$ edges that can be deleted in any order to yield a refining sequence from $r_1$ to $r_k$. This means that

    \begin{equation}\label{eq:count_one}
        N^*(r_k) = N(r_k)(k-1)!.
    \end{equation}

    Second, we count $N^*(r_k)$ starting from $r_k$ and adding edges. A rooted forest $r_{k-1}$ that refines $r_k$ can be obtained by adding a directed edge from any of the $n$ nodes to any of the $k-1$ roots of rooted trees that do not contain that node. Since there are $n(k-1)$ possibilities at each step, continuing for $k-1$ steps gives
    
    \begin{equation}\label{eq:count_two}
        N^*(r_k) = n^{k-1}(k-1)!.
    \end{equation}

    Equating Eq.~\ref{eq:count_one} and Eq.~\ref{eq:count_two}, we obtain $N(r_k) = n^{k-1}$.
\end{proof}

Note that $r_n$ is just the set of $n$ isolated nodes. The number of all rooted trees on $n$ nodes is therefore $N(r_n) = n^{n-1}$. This number is greater than the number of all trees by a factor of $n$.

\begin{corollary}[Cayley's formula]\label{thm:cayleys_formula}
    The number of trees on $n$ nodes is $T_n = n^{n-2}$.
\end{corollary}

Cayley's formula is a famous result with many different proofs \citep[e.g.][]{prufer1918neuer, aigner2014proofs}.

We next straightforwardly apply Theorem~\ref{thm:pitman_lemma1} to count unrooted trees.

\begin{theorem}[{\citep[][3]{pitman1999coalescent}}]\label{thm:pitman_lemma3}
    For each forest $f_k \in \mathcal{F}_{n,k}$, containing trees with sizes $n_1, \dots, n_k$ such that $\sum n_i = n$, the number of trees that contain $f_k$ is
    
    \begin{equation}\label{eq:tree_count}
        N(f_k) = \left( \prod_{i=1}^k n_i \right) n^{k-2}.
    \end{equation}
\end{theorem}

\begin{proof}
    Let $f_k \in \mathcal{F}_{n,k}$ be a fixed forest. We count the rooted trees that contain $f_k$ (ignoring edge directions) in two ways. First, there are $n$ ways to choose a tree's root, giving us $n N(f_k)$ rooted trees. Second, choosing a root for each tree component of $f_k$ and applying Theorem~\ref{thm:pitman_lemma1} gives us $\left(\prod_i n_i \right) n^{k-1}$ rooted trees. Equating these expressions yields the above result.
\end{proof}

We now observe a relation between edge deletion in a uniform random tree, an additive coalescent process among trees, and a uniform random forest.

\begin{theorem}[{\citep[][5]{pitman1999coalescent}}]\label{thm:pitman_theorem5}
    The following three descriptions (i), (ii), and (iii) for the distribution of a sequence of forests $(F_1, \dots, F_n)$ are equivalent, and imply that, for each $1 \leq k \leq n$ and for each  $f_k \in \mathcal{F}_{n,k}$ with tree components of size $n_1, \dots, n_k$ in some arbitrary order, 
    \begin{equation}
    P(F_k = f_k) = \frac{\prod_{i=1}^k n_i}{n^{n-k}\binom{n-1}{k-1}}
    \label{eq:Fk-dist}
    \end{equation}
    \begin{itemize}
      \item[(i)] $F_1$ is a uniform random tree in $\mathcal{T}_n$, and indexing the edges $e_j$ of $F_1$ by a uniform random permutation of $1, \dots, n - 1$, for each $1 \leq k \leq n$ the forest $F_k$ is derived from $F_1$ by deleting the edges $e_1, \dots, e_{k-1}$;
      
      \item[(ii)] $F_n$ is the trivial forest, and for $n \geq k \geq 2$, given $(F_n, \dots, F_k)$ where $F_k$ has $k$ tree components $T_i, \dots, T_k$ with sizes $n_1, \dots, n_k$ such that $\sum_{i=1}^k n_i = n$, the forest $F_{k-1} \in \mathcal{F}_{n,k-1}$ is obtained by adding an edge $\{a,b\}$ to $F_k$, chosen according to the following rule: first pick trees $T_i$ and $T_j$, where $1 \leq i < j \leq k$, with probability
      \begin{equation}
        P(i, j) = \frac{n_i + n_j}{n(k-1)},
        \label{eq:additive-rule}
      \end{equation}
      then pick $a$ and $b$ independently and uniformly at random from $T_i$ and $T_j$ respectively;
      
      \item[(iii)] The sequence $(F_1, \dots, F_n)$ has uniform distribution over the set of all $(n-1)!~n^{n-2}$ refining sequences of forests $(f_1, \dots, f_n)$ such that $f_k \in \mathcal{F}_{n,k}$ for every $1 \leq k \leq n-1$.
    \end{itemize}
\end{theorem}

\begin{proof}
    In both descriptions (i) and (iii), the forest $F_k$ is determined by the choice of a tree $t \in T_n$ and a subset of $k-1$ edges of $t$. It follows from Cayley's formula \ref{thm:cayleys_formula} that in both cases there is a total number of $n^{n-2}\binom{n-1}{k-1}$ equally likely choices, so descriptions (i) and (iii) are equivalent. The probability that $F_k = f_k$ equals the number of trees that contain $f_k$  (Eq.~\ref{eq:tree_count}) divided by the total number, giving the probability shown in Eq.~\ref{eq:Fk-dist}.

    Next, we show that description (ii) is equivalent to (i). The process is Markov, so it suffices to consider the conditional probability that $F_{k-1} = f_{k-1}$ given $F_k = f_k$. One approach is as follows \citep{sheth1997coagulation}. Using Bayes' rule,

    \begin{equation}
        P(F_{k-1} = f_{k-1}~|~F_k = f_k) = \frac{P(F_{k-1} = f_{k-1})}{P(F_k = f_k)} P(F_k = f_k~|~F_{k-1} = f_{k-1}).
    \end{equation}
    
    From Eq.~\ref{eq:Fk-dist} and the fact that in (i), $f_k$ is obtained from $f_{k-1}$ by deleting one of $n - (k - 1)$ edges uniformly at random, we have

    \begin{align}
        P(F_{k-1} = f_{k-1}~|~F_k = f_k) &= \frac{n^{n-k}\binom{n-1}{k-1}(n_i+n_j)\prod_{l \notin \{i, j\}}n_l}{n^{n-(k-1)}\binom{n-1}{k-2}\prod_l n_l}\frac{1}{n - (k-1)}\nonumber\\
        &= \frac{(n-(k-1))(n_i + n_j)}{n(k-1) n_i n_j}\frac{1}{n - (k-1)}\nonumber\\
        &= \frac{n_i + n_j}{n(k-1)} \frac{1}{n_i} \frac{1}{n_j}.\label{eq:conditional_probability}
    \end{align}

    This conditional probability is satisfied by the procedure in description (ii).
\end{proof}

\subsection{Cyclic coalescent}\label{sec:theorems_cyclic}

We now build on these prior results to show that a distinct construction, the cyclic coalescent, is equivalent to the descriptions in Theorem~\ref{thm:pitman_theorem5}. This implies that running the cyclic coalescent to completion produces a uniform random tree. Moreover, such an algorithm can be implemented in almost linear time, as discussed in the main text.

\begin{theorem}[Cyclic coalescent]\label{thm:cyclic_coalescent}
    The following description for the distribution of a sequence of forests $(F_1, \dots F_n)$ is equivalent to the descriptions in Theorem~\ref{thm:pitman_theorem5}.
    
    $F_n$ is the trivial forest, and arranging the tree components (nodes) of $F_n$ in a cycle indexed by a uniform random permutation of $1, \dots, n$, for $n \ge k \ge 2$ and given $(F_n, \dots, F_k)$ where $F_k$ has $k$ tree components $T_i, \dots, T_k$ with sizes $n_1, \dots, n_k$ such that $\sum_{i=1}^k n_i = n$, the forest $F_{k-1} \in \mathcal{F}_{n,k-1}$ is obtained by adding an edge $\{a, b\}$ to $F_k$, chosen according to the following rule: first pick the node $a$ uniformly at random from $F_k$, then, denoting the tree containing this first-chosen node as $T_i$ and its successor tree as $T_j$, pick the node $b$ uniformly at random from $T_j$.
\end{theorem}

\begin{proof}
    Consider the conditional probability $F_{k-1} = f_{k-1}$ given $F_k = f_k$, where $f_{k-1}$ differs from $f_k$ only by adding the edge $\{a, b\}$. The edge $\{a, b\}$ can be added iff $a \in T_i$ and $b \in T_j$ or $b \in T_i$ and $a \in T_j$. There are $k$ trees, so both arrangements occur with equal probability $1/(k-1)$. Then because the first-chosen node is uniformly distributed among $F_k$ and among $T_i$, and the second-chosen node is uniformly distributed among $T_j$, the total probability to choose $\{a, b\}$ is

    \begin{align}
        P(F_{k-1} = f_{k-1}~|~F_k = f_k) &=
        \frac{1}{k-1}\frac{n_i}{n}\frac{1}{n_i}\frac{1}{n_j} + \frac{1}{k-1}\frac{n_j}{n}\frac{1}{n_j}\frac{1}{n_i}\nonumber\\
        &= \frac{n_i + n_j}{n(k-1)}\frac{1}{n_i}\frac{1}{n_j}.
    \end{align}

    This matches the conditional probability given by Eq.~\ref{eq:conditional_probability}.
\end{proof}

\begin{corollary}
    The cyclic coalescent after $n - 1$ steps produces a uniform random tree in $\mathcal{T}_n$. 
\end{corollary}

This follows from description (i) in Theorem~\ref{thm:pitman_theorem5}.

\begin{corollary}
    The cyclic coalescent after $k$ steps produces a uniform distribution over forests $f_k \in \mathcal{F}_{n,k}$. 
\end{corollary}

This follows from description (iii) in Theorem~\ref{thm:pitman_theorem5}.

\section{Additional probing experiments}\label{appendix:additional_experiments}

Fig.~\ref{fig:sum_zu_MSE_probes} shows the results of linear probing experiments on the partial cumulative sums of latent values.

\begin{figure}[htbp]
    \centering
    \begin{subfigure}[b]{0.45\textwidth}
        \centering
        \includegraphics[trim={0.2cm 0.3cm 2cm 2cm}, clip, width=\textwidth]{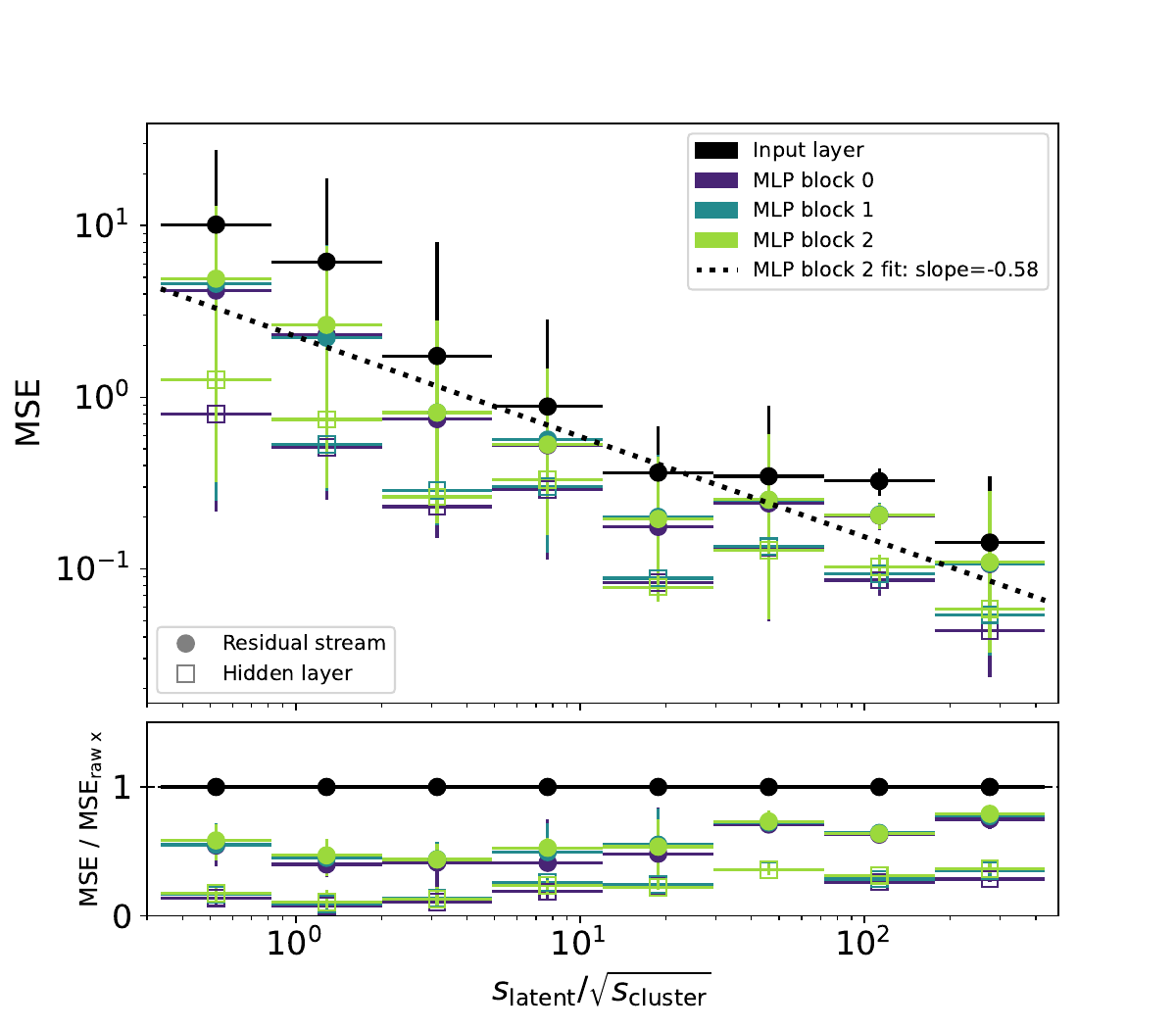}
        \caption{}
        \label{subfig:1clusterMSE_sumzu}
    \end{subfigure}
    \hfill
    \begin{subfigure}[b]{0.45\textwidth}
        \centering
        \includegraphics[trim={0.2cm 0.3cm 2cm 2cm}, clip, width=\textwidth]{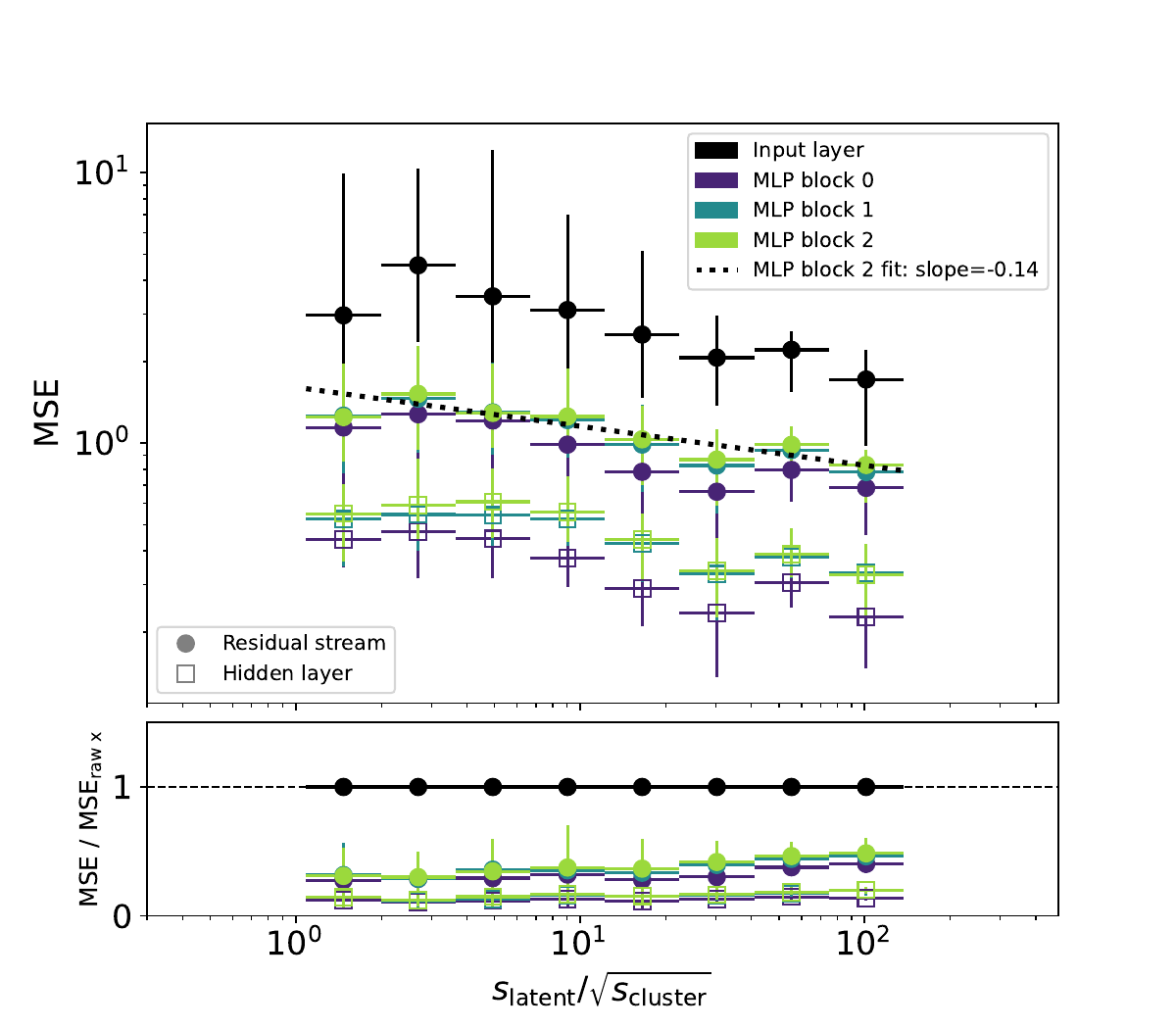}
        \caption{}
        \label{subfig:multiclusterMSE_sumzu}
    \end{subfigure}
    \caption{Per-latent linear probe performance for partial cumulative sums of latent values $z_i$, showing (a) the one-cluster dataset and (b) the multi-cluster dataset. Dots (squares) mark the median MSE for the residual stream (hidden activations). Error bars show 25th and 75th percentiles. The ratio panels show each layer's per-latent MSE normalized by the MSE of probes trained on the raw input.}
    \label{fig:sum_zu_MSE_probes}
\end{figure}

\section{Additional validation plots}

Fig.~\ref{fig:ccdf_clustersize} shows the complementary cumulative distribution function (CCDF) for the multi-cluster dataset described in Sec.~\ref{sec:experiments}. The fitted slope is consistent with the theoretical expectation of $3/2$ as derived in Eq.~\ref{eq:ccdf}.

\begin{figure}[htbp]
    \centering
    \includegraphics[width=0.8\linewidth]{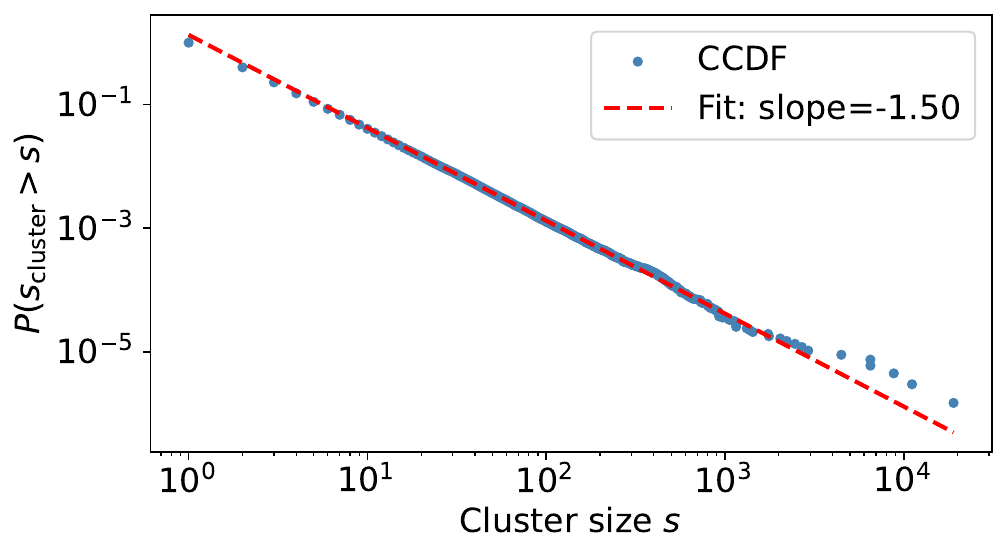}
    \caption{Cluster size distribution. A maximum-likelihood power-law fit to clusters with $s \geq 35$ is shown as a dashed red line. Each cluster (not data point) is counted once in the distribution.}
    \label{fig:ccdf_clustersize}
\end{figure}

Fig.~\ref{fig:correctness_validation} shows plots validating the correctness of the cyclic coalescent algorithm. A well-known bijection, called a Pr\"ufer sequence \citep{prufer1918neuer}, exists between the set of $n^{n-2}$ labeled trees on $n$ nodes and the set of sequences of length $n-2$ on the labels 1 to $n$. This bijection permits a powerful check of uniformity tractable at small tree sizes. Fig.~\ref{fig:prufer_correctness} confirms that trees of size $n=6$ generated by the cyclic coalescent have a uniform distribution when coded as Pr\"ufer sequences. As another check, Fig.~\ref{fig:degree_distribution} validates that a large tree of size $10^6$ has the theoretically expected degree distribution $1 + \mathrm{Poisson}(1)$.

\begin{figure}[htbp]
    \centering
    \begin{subfigure}[t]{0.48\textwidth}
        \centering
        \includegraphics[width=\textwidth]{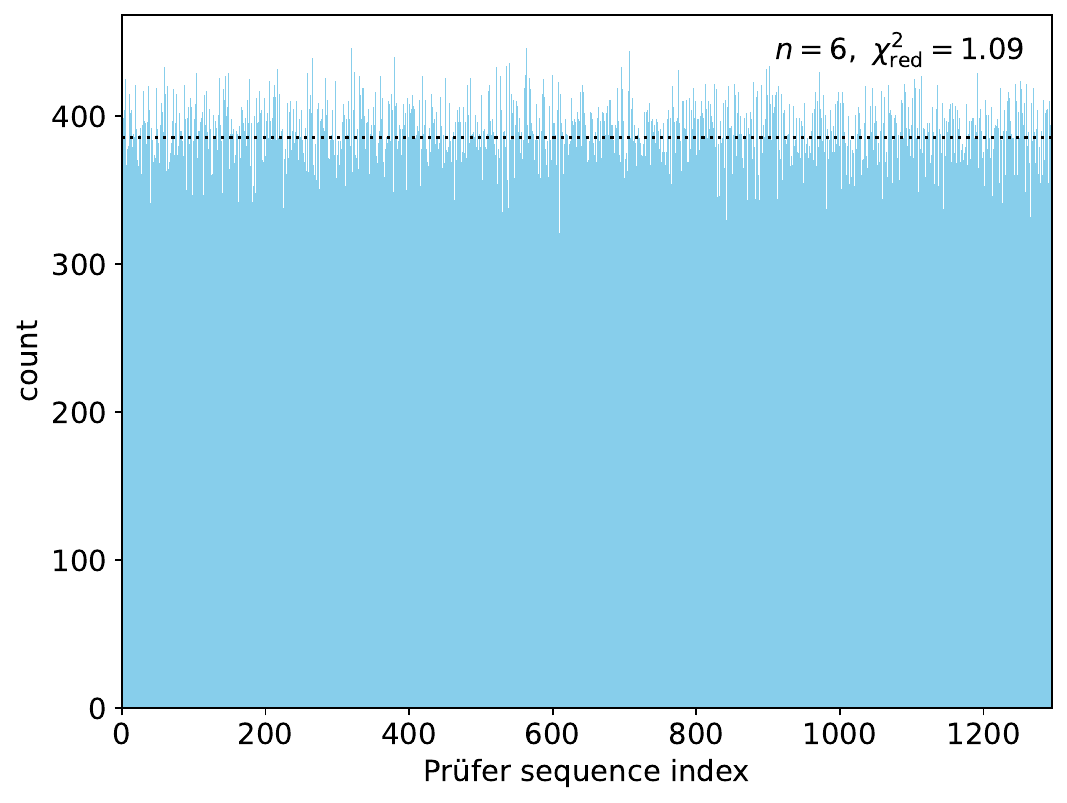}
        \caption{Distribution over Pr\"ufer sequences.}
        \label{fig:prufer_correctness}
    \end{subfigure}
    \hfill
    \begin{subfigure}[t]{0.48\textwidth}
        \centering
        \includegraphics[width=\textwidth]{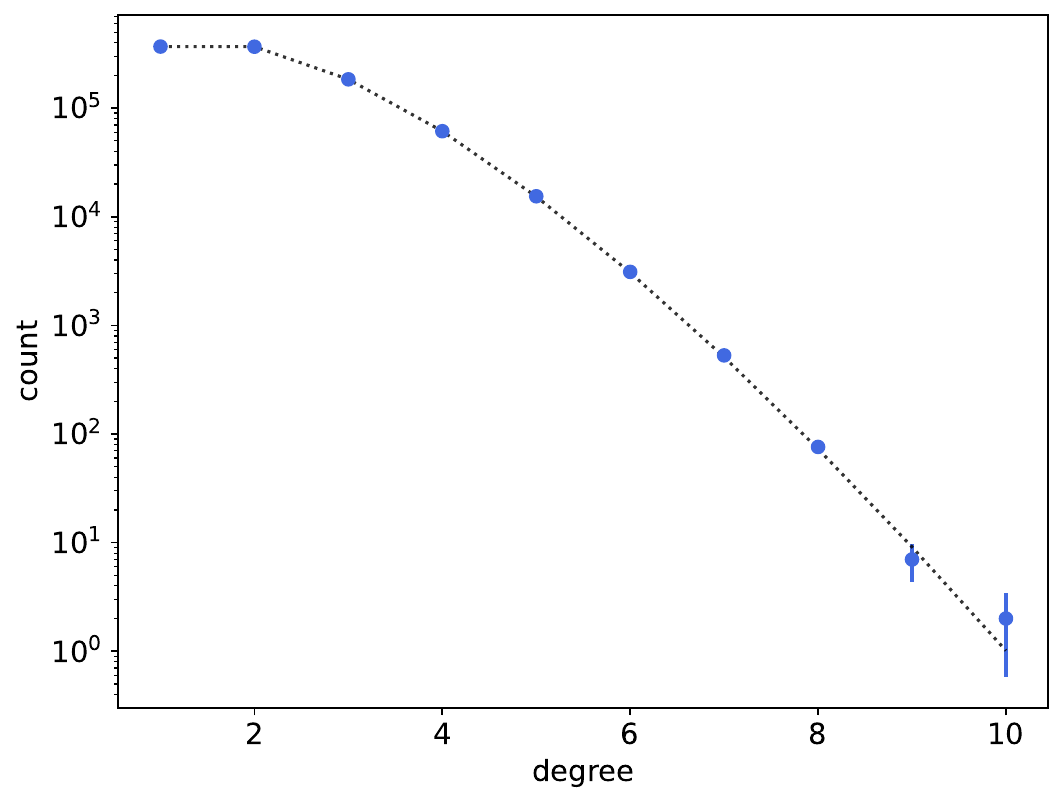}
        \caption{Degree distribution.}
        \label{fig:degree_distribution}
    \end{subfigure}
    \caption{Correctness validation for the cyclic coalescent. Panel~(\subref{fig:prufer_correctness}) shows that 500,000 trees generated with the cyclic coalescent have a uniform distribution over the 1296 Pr\"ufer sequences for size $n=6$. Panel~(\subref{fig:degree_distribution}) shows the degree distribution of a tree of size $10^6$. Measured counts are shown in blue with Poisson error bars. The empirical counts match the theoretical distribution, shown with a black dotted line.}
    \label{fig:correctness_validation}
\end{figure}

In Fig.~\ref{fig:timing_validation}, we show timing measurements to benchmark the empirical performance of the cyclic coalescent algorithm implemented in Python. As shown in Fig.~\ref{fig:time_measurement}, the measured time to generate a tree of size $n$ scales slightly superlinearly, with a measured slope of about 1.1, rather than 1 as expected for a linear-time algorithm. This observation is consistent with $O(n\log n)$ scaling. We believe that the theoretical time complexity is obscured in practice by details of the Python implementation, rather than a fundamental issue. Specifically, the tree is constructed by iteratively inserting edges into adjacency dictionaries associated with each node. As the tree is built up using random-access edge insertions, an extraneous logarithmic slowdown may result due to cache misses among these heap-allocated dictionaries.

To see this, we compare the performance to two baselines with known $O(n)$ time complexity. First, we compare to \texttt{random\_labeled\_tree()} from the \texttt{networkx} library \citep{SciPyProceedings_11}, which constructs a random tree from a Pr\"ufer sequence using the optimal $O(n)$ algorithm from \citep{wang2009optimal}. Second, we compare to a trivial baseline with transparently linear algorithmic time complexity, shown below as the function \texttt{build\_random\_graph()}.

\begin{lstlisting}[style=py]
import networkx as nx
import numpy as np

def build_random_graph(n: int, rng: np.random.Generator):
    G = nx.empty_graph(n)
    edges = rng.integers(0, n, size=(n - 1, 2))
    for u, v in edges:
        G.add_edge(int(u), int(v))
    return G
\end{lstlisting}

As can be seen in Fig.~\ref{fig:time_measurement}, the slopes of all three methods are consistent within error. The reported error bars are underestimates, as they do not incorporate variation in measured wall-clock time from run to run. In Fig.~\ref{fig:time_ratio}, we show the ratio of the measured time of the cyclic coalescent to the \texttt{build\_random\_graph()} baseline. The ratios appear roughly constant and have no apparent trend, consistent with equivalent to linear performance. Running the cyclic coalescent takes about $60\%$ longer than a pure graph-construction baseline. This is expected because the cyclic coalescent constructs an additional binary coalescence tree containing $O(n)$ nodes alongside the random tree.

\begin{figure}[htbp]
    \centering
    \begin{subfigure}[t]{0.48\textwidth}
        \centering
        \includegraphics[width=\textwidth]{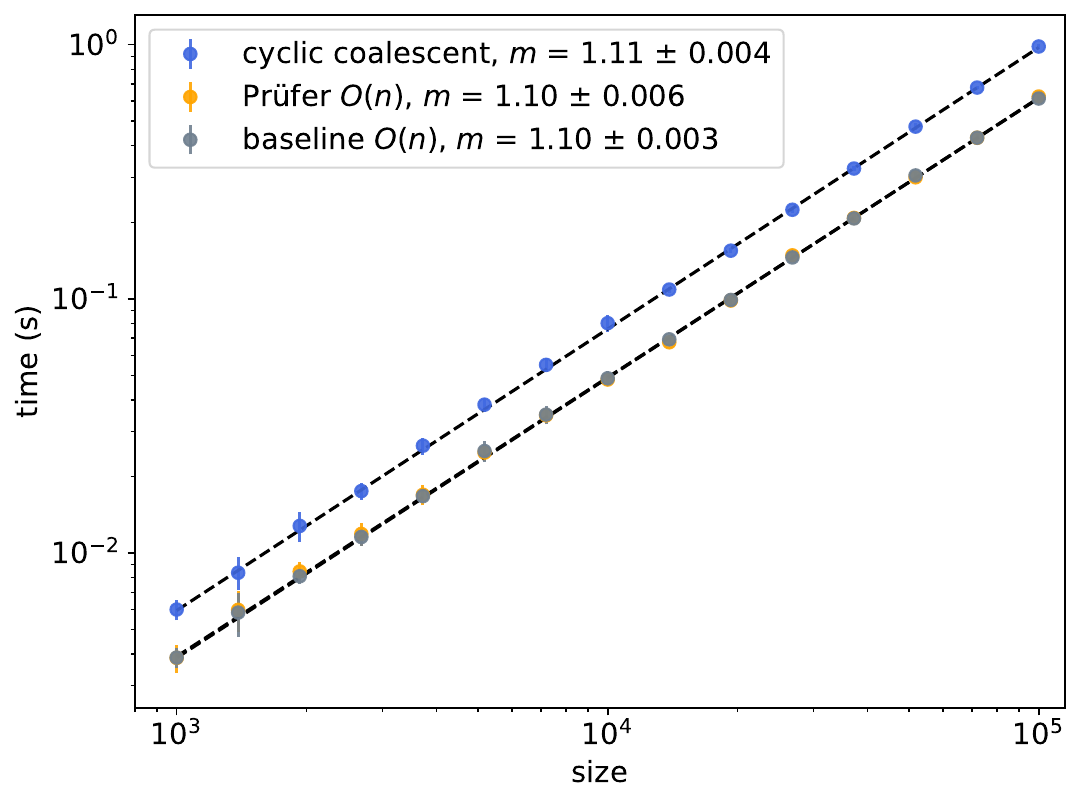}
        \caption{Absolute time required to generate one cluster.}
        \label{fig:time_measurement}
    \end{subfigure}
    \hfill
    \begin{subfigure}[t]{0.48\textwidth}
        \centering
        \includegraphics[width=\textwidth]{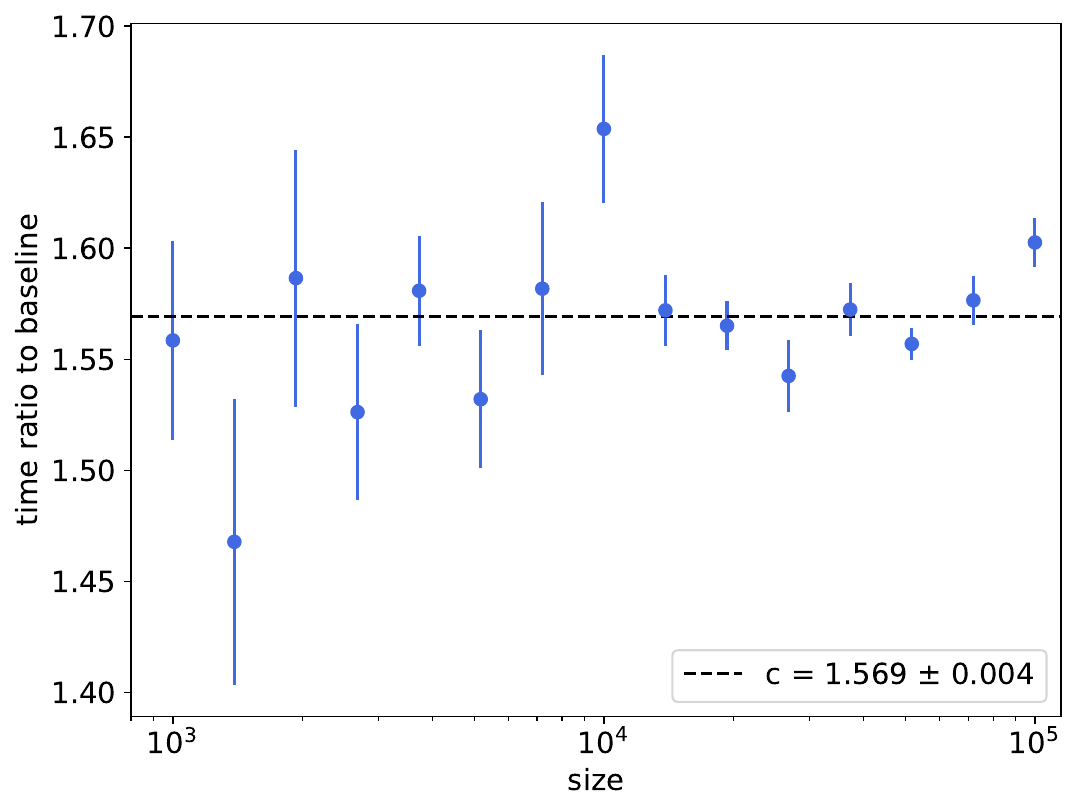}
        \caption{Relative time required to generate one cluster.}
        \label{fig:time_ratio}
    \end{subfigure}
    \caption{Performance validation for the cyclic coalescent. Panel~(\subref{fig:time_measurement}) shows the time required to generate a random tree as a function of tree size for the cyclic coalescent and two $O(n)$ baseline methods, described in the text. The data points show the mean and standard deviation over 20 trials. The black dashed lines show linear fits. The slopes are consistent among the methods. Panel~(\subref{fig:time_ratio}) shows the performance ratio of the cyclic coalescent to the linear baseline \texttt{build\_random\_graph()} as a function of tree size. No clear trend is apparent. The black dashed line shows a constant fit.}
    \label{fig:timing_validation}
\end{figure}

\end{document}